\documentclass[10pt]{article}

\usepackage{etoolbox}
\usepackage{times}
\newcommand{\arxiv}[1]{\iftoggle{neurips}{}{#1}}
\newcommand{\neurips}[1]{\iftoggle{neurips}{#1}{}}
\newtoggle{neurips}
\global\toggletrue{neurips}
\global\togglefalse{neurips}

\neurips{
\PassOptionsToPackage{dvipsnames}{xcolor} 
\usepackage[preprint]{neurips_2024}
}

\newcommand{\loose}{\looseness=-1}

\usepackage[utf8]{inputenc} 
\usepackage[T1]{fontenc}    
\usepackage{url}            
\usepackage{booktabs}       
\usepackage{amsfonts}       
\usepackage{nicefrac}       
\usepackage{microtype}      

\usepackage{tocloft}            

\usepackage{enumitem}
\neurips{\setlist[itemize]{noitemsep, topsep=0pt}}
\arxiv{\setlist[itemize]{}}

\usepackage{breakcites}
\usepackage{makecell}

\newtoggle{draft}
\togglefalse{draft}

\usepackage{mathrsfs}

\usepackage{algorithm}
\usepackage{verbatim}
\usepackage[noend]{algpseudocode}

\usepackage{multicol}

\usepackage{colortbl}

\usepackage{setspace}

\usepackage{transparent}

\usepackage{inconsolata}
\usepackage[scaled=.90]{helvet}
\usepackage{xspace}

\usepackage{pifont}

\usepackage{multirow}

\arxiv{
\usepackage[letterpaper, left=1in, right=1in, top=1in, bottom=1in]{geometry}
\PassOptionsToPackage{hypertexnames=false}{hyperref}  
\usepackage{parskip}
\usepackage[dvipsnames]{xcolor}
\usepackage[colorlinks]{hyperref}
\hypersetup{
    citecolor=[RGB]{50,100,170},
    linkcolor=[RGB]{50,100,170},
    urlcolor=[RGB]{255,102,178}}
}

\neurips{
\usepackage[colorlinks]{hyperref}
\hypersetup{
    citecolor=[RGB]{50,100,170},
    linkcolor=[RGB]{50,100,170},
    urlcolor=[RGB]{255,102,178}}

}

\usepackage{microtype}
\usepackage{hhline}

\makeatletter
\newcommand{\neutralize}[1]{\expandafter\let\csname c@#1\endcsname\count@}
\makeatother

\usepackage{algorithm}

\arxiv{
\usepackage{natbib}
\bibliographystyle{plainnat}
\bibpunct{(}{)}{;}{a}{,}{,}
}

\usepackage{amsthm}
\usepackage{mathtools}
\usepackage{amsmath}
\usepackage{bbm}
\usepackage{amsfonts}
\usepackage{amssymb}

\usepackage{xpatch}

\usepackage{thmtools}
\usepackage{thm-restate}
\declaretheorem[name=Theorem,parent=section]{theorem}
\declaretheorem[name=Lemma,parent=section]{lemma}
\declaretheorem[name=Assumption, parent=section]{assumption}
\declaretheorem[name=Condition, parent=section]{condition}
\declaretheorem[qed=$\triangleleft$,name=Example,parent=section]{example}
\declaretheorem[name=Remark,style=definition, parent=section]{remark}
\declaretheorem[name=Proposition, parent=section]{proposition}

\makeatletter
  \renewenvironment{proof}[1][Proof]%
  {%
   \par\noindent{\bfseries\upshape {#1.}\ }%
  }%
  {\qed\newline}
  \makeatother

\theoremstyle{definition}  

\newtheorem{corollary}{Corollary}[section]

\theoremstyle{plain}
\newtheorem{definition}{Definition}[section]

\xpatchcmd{\proof}{\itshape}{\normalfont\proofnameformat}{}{}
\newcommand{\proofnameformat}{\bfseries}

\usepackage[nameinlink,capitalize]{cleveref}

\newcommand{\pref}[1]{\cref{#1}}

\newcommand{\savehyperref}[2]{\texorpdfstring{\hyperref[#1]{#2}}{#2}}

\renewcommand{\eqref}[1]{\texorpdfstring{\hyperref[#1]{(\ref*{#1})}}{(\ref*{#1})}}

\crefformat{equation}{#2Eq. (#1)#3}
\Crefformat{equation}{#2Eq. (#1)#3}

\Crefformat{figure}{#2Figure #1#3}
\Crefname{assumption}{Assumption}{Assumptions}
\Crefformat{assumption}{#2Assumption #1#3}
\Crefname{subsubsection}{Section}{Sections}
\crefformat{subsubsection}{#2Section #1#3}
\Crefformat{subsubsection}{#2Section #1#3}

\usepackage{crossreftools}
\usepackage{xparse}

\ExplSyntaxOn
\DeclareDocumentCommand{\XDeclarePairedDelimiter}{mm}
 {
  \__egreg_delimiter_clear_keys: 
  \keys_set:nn { egreg/delimiters } { #2 }
  \use:x 
   {
    \exp_not:n {\NewDocumentCommand{#1}{sO{}m} }
     {
      \exp_not:n { \IfBooleanTF{##1} }
       {
        \exp_not:N \egreg_paired_delimiter_expand:nnnn
         { \exp_not:V \l_egreg_delimiter_left_tl }
         { \exp_not:V \l_egreg_delimiter_right_tl }
         { \exp_not:n { ##3 } }
         { \exp_not:V \l_egreg_delimiter_subscript_tl }
       }
       {
        \exp_not:N \egreg_paired_delimiter_fixed:nnnnn 
         { \exp_not:n { ##2 } }
         { \exp_not:V \l_egreg_delimiter_left_tl }
         { \exp_not:V \l_egreg_delimiter_right_tl }
         { \exp_not:n { ##3 } }
         { \exp_not:V \l_egreg_delimiter_subscript_tl }
       }
     }
   }
 }

\keys_define:nn { egreg/delimiters }
 {
  left      .tl_set:N = \l_egreg_delimiter_left_tl,
  right     .tl_set:N = \l_egreg_delimiter_right_tl,
  subscript .tl_set:N = \l_egreg_delimiter_subscript_tl,
 }

\cs_new_protected:Npn \__egreg_delimiter_clear_keys:
 {
  \keys_set:nn { egreg/delimiters } { left=.,right=.,subscript={} }
 }

\cs_new_protected:Npn \egreg_paired_delimiter_expand:nnnn #1 #2 #3 #4
 {
  \mathopen{}
  \mathclose\c_group_begin_token
   \left#1
   #3
   \group_insert_after:N \c_group_end_token
   \right#2
   \tl_if_empty:nF {#4} { \c_math_subscript_token {#4} }
 }
\cs_new_protected:Npn \egreg_paired_delimiter_fixed:nnnnn #1 #2 #3 #4 #5
 {
  \mathopen{#1#2}#4\mathclose{#1#3}
  \tl_if_empty:nF {#5} { \c_math_subscript_token {#5} }
 }
\ExplSyntaxOff

\XDeclarePairedDelimiter{\supnorm}{
  left=\lVert,
  right=\rVert,
  subscript=\infty
  }

\usepackage{mkolar_definitions}

\newcommand{\supp}{\mathsf{supp}}

\newcommand{\mathand}{\quad\text{and}\quad}

\newcommand{\piref}{\pi_{\mathsf{ref}}}
\newcommand{\KL}{\mathsf{KL}}
\newcommand{\pirlhf}{\pi_{\mathsf{rlhf}}}
\newcommand{\pidpo}{\pi_{\mathsf{dpo}}}
\newcommand{\piipo}{\pi_{\mathsf{ipo}}}
\newcommand{\pihypo}{\pi_{\mathsf{hypo}}}

\newcommand{\cglobal}{C_{\mathsf{glo}}}
\newcommand{\edpo}{\varepsilon_{\mathsf{dpo}}}
\newcommand{\ereward}{\varepsilon_{\mathsf{reward}}}
\newcommand{\ekl}{\varepsilon_{\mathsf{kl}}}
\newcommand{\ldpo}{\ell_{\mathsf{dpo}}}
\newcommand{\lipo}{\ell_{\mathsf{ipo}}}
\newcommand{\Rdpo}{\Rcal_{\mathsf{dpo}}}
\newcommand{\rdpo}{\widehat{r_{\mathsf{dpo}}}}
\newcommand{\ehypo}{\varepsilon_{\mathsf{hypo}}}

\newcommand{\rhypo}{\widehat{r_{\mathsf{hypo}}}}

\newcommand*\colourcheck[1]{%
  \expandafter\newcommand\csname #1check\endcsname{\textcolor{#1}{\ding{52}}}%
}
\colourcheck{yellow}
\colourcheck{green}
\colourcheck{red}

\newcommand*\colourquestion[1]{%
  \expandafter\newcommand\csname #1question\endcsname{\textcolor{#1}{\textbf{?}}}%
}

\colourquestion{red}

\neurips{
  \renewcommand{\citet}{\cite}
}

\usepackage{color-edits}
 \addauthor{jab}{ForestGreen}
 \addauthor{ys}{TealBlue}
 \addauthor{as}{BurntOrange}
 \addauthor{ws}{red}

 \arxiv{
\usepackage[final]{showlabels}

}

\makeatletter
\let\OldStatex\Statex
\renewcommand{\Statex}[1][3]{%
  \setlength\@tempdima{\algorithmicindent}%
  \OldStatex\hskip\dimexpr#1\@tempdima\relax}
\makeatother

\usepackage{accents}
\usepackage{wrapfig}
\usepackage{tikz}
\usetikzlibrary{decorations.pathreplacing}

 \addtocontents{toc}{\protect\setcounter{tocdepth}{0}}

\let\oldparagraph\paragraph
\renewcommand{\paragraph}[1]{\oldparagraph{#1.}}

\algrenewcommand\algorithmicrequire{\textbf{require}}

\usepackage{yfonts}
\title{The Importance of Online Data: \\Understanding Preference Fine-tuning via Coverage}
\neurips{
\author{%
  David S.~Hippocampus\thanks{Use footnote for providing further information
    about author (webpage, alternative address)---\emph{not} for acknowledging
    funding agencies.} \\
  Department of Computer Science\\
  Cranberry-Lemon University\\
  Pittsburgh, PA 15213 \\
  \texttt{hippo@cs.cranberry-lemon.edu} \\
}

}

\arxiv{
    
\author{
      Yuda Song$^1$ \; Gokul Swamy$^1$ \; Aarti Singh$^1$ \; J. Andrew Bagnell$^{1,2}$ \; Wen Sun$^3$\\
      \vspace{-2mm} \\
      \normalsize{$^1$Carnegie Mellon University \qquad $^2$Aurora Innovation
      \qquad $^3$ Cornell University }\\
      \vspace{-2mm} \\
      \normalsize{\texttt{\{yudas,gswamy,aarti\}@cs.cmu.edu},\; \texttt{dbagnell@aurora.tech},\; \texttt{ws455@cornell.edu}}
    }
\date{}
}

\begin{document}

\maketitle
\neurips{\vspace{-5mm}}
\begin{abstract}
  Learning from human preference data has emerged as the dominant paradigm for fine-tuning large language models (LLMs). The two most common families of techniques -- online reinforcement learning (RL) such as Proximal Policy Optimization (PPO) and offline contrastive methods such as Direct Preference Optimization (DPO) -- were positioned as equivalent in prior work due to the fact that both have to start from the same offline preference dataset. To further expand our theoretical understanding of the similarities and differences between online and offline techniques for preference fine-tuning, we conduct a rigorous analysis through the lens of \textit{dataset coverage}, a concept that captures how the training data covers the test distribution and is widely used in RL. We prove that a global coverage condition is both necessary and sufficient for offline contrastive methods to converge to the optimal policy, but a weaker partial coverage condition suffices for online RL methods. This separation provides one explanation of why online RL methods can perform better than offline methods, especially when the offline preference data is not diverse enough. Finally, motivated by our preceding theoretical observations, we derive a hybrid preference optimization (HyPO) algorithm that uses offline data for contrastive-based preference optimization and online unlabeled data for KL regularization. Theoretically and empirically, we demonstrate that HyPO is more performant than its pure offline counterpart DPO, while still preserving its computation and memory efficiency. \looseness=-1
\end{abstract}

\neurips{\vspace{-3mm}}
\section{Introduction}
\neurips{\vspace{-3mm}}
Due to the difficulty of manually specifying reward functions for complex tasks \citep{casper2023open}, preference-based learning has emerged as a critical component in the fine-tuning procedure for large language models (LLMs) \citep{stiennon2020learning, ouyang2022training, touvron2023llama, team2023gemini}.
There are two predominant flavors of preference learning for LLMs: online reinforcement learning (RL) methods such as PPO \citep{christiano2017deep, ouyang2022training} 
and offline contrastive methods like Direct Preference Optimization (DPO) \citep{rafailov2024direct} and Identity Preference Optimization (IPO) \citep{azar2024general}.

Online RL methods usually follow the two-stage procedure prescribed in \citet{ouyang2022training}: one first trains a reward model (classifier) on a fixed offline preference dataset before using it to provide reward labels for on-policy generations, which are then fed to a downstream RL algorithm like Proximal Policy Optimization (PPO) \citep{schulman2017proximal}. Since the reward model is learned from static offline preference data, to avoid over-optimizing the reward model \citep{gao2023scaling}, one typically adds a reverse KL penalty to encourage the model to stay close to some reference policy. We will refer to this procedure as reinforcement learning from human feedback (RLHF) in this paper. While empirically performant, RLHF requires repeated querying of the reward model (which is often itself an LLM) as well as sampling from the current policy.
In response to the computational expense and relatively complex nature of this procedure, purely offline methods like DPO \citep{rafailov2024direct} and IPO \citep{azar2024general} have been proposed as alternative methods for preference fine-tuning. These methods do not need to fit separate reward models, instead opting to simply train the policy directly on the offline preference dataset via a ranking loss.

Offline contrastive methods like DPO are usually derived via applying a reparameterization trick to the closed-form solution of the minimum relative entropy problem \citep{ziebart2008maximum} that RLHF techniques attempt to approximate. Thus, several authors have described these methods as equivalent (at least in theory) to the standard RLHF procedure \citep{rafailov2024direct,azar2024general}. However, recent (mostly empirical) work has contradicted this perspective: \citet{ tang2024understanding} find that online methods out-perform offline methods and attribute this fundamentally to on-policy sampling, \citet{xu2024dpo} argues that the online RL methods produce an often desirable subset of the possible DPO loss minimizers, and \citet{tajwar2024preference} provide empirical support for the claim that online and contrastive training provide orthogonal benefits.
However, a rigorous theoretical separation is still lacking in the pre-existing literature, which motivates our key questions:
\begin{quote}
  \begin{center}
        \emph{What is the statistical separation between the online RLHF method and offline contrastive methods? What causes this separation and what does it imply?
    }
    \end{center}
\end{quote}

To answer these questions, we focus on the coverage of the preference dataset, a key concept that is widely used in RL \citep{kakade2002approximately,munos2008finite,zhan2022offline} for analyzing the impact of offline or exploratory data distributions. Through the lens of coverage of the offline preference dataset, we make the following contributions: \looseness=-1

\begin{itemize}

\item We prove that the global coverage condition \citep{munos2008finite}, the strongest possible coverage condition in RL, is necessary for offline contrastive algorithms like DPO to converge to the optimal policy. In contrast, we identify 
  a weaker local coverage condition that is sufficient for online RLHF algorithms,
  thus provably separating the two types of algorithms. 
The separation is due to the difference in reward  
modeling and on/offline regularization -- in short, \textit{there is no free lunch
from bypassing explicit reward learning and online rollouts}. As global coverage might sometimes be violated in practice, our separation result can perhaps explain why RLHF works better than offline methods \citep{tajwar2024preference, tang2024understanding, yuan2024advancing}.

\item 
Although offline contrastive methods are derived from a reverse-KL objective, we prove that the policies trained via offline methods can still have infinite reverse-KL in the partial coverage setting. In contrast, we show that RLHF can always control the reverse KL via directly optimizing reverse KL using online samples. This means that on realistic problems, RLHF has stronger guarantees for remaining close to the reference policy than offline contrastive methods.

\item 
We propose Hybrid Preference Optimization (HyPO) to address the deficiencies of offline contrastive methods while maintaining some of their computational simplicity. HyPO is a hybrid RL algorithm \citep{xie2021policy,song2022hybrid} where offline data is used for the DPO objective while online samples are used to explicitly control the reverse KL divergence to the reference policy. We empirically demonstrate that HyPO outperforms DPO, on the \emph{TL;DR} summarization task \citep{stiennon2020learning} on all metrics including both the GPT4 win-rate and the reverse KL divergence to the reference policy, and on general chat benchmarks such as AlpacaEval 2.0 \citep{dubois2024length}, trained with the UltraFeedback dataset \citep{cui2023ultrafeedback}. In addition, HyPO also mitigates the overfitting issues observed in the offline constrastive based methods \citep{tang2024understanding}.

\item 
We provide an explanation of why RLHF and offline contrastive methods decrease the probability of both preferred and rejected responses during training.
In particular, under our function approximation-based global coverage condition, we show that such behavior is actually desirable for DPO and RLHF policies to extrapolate and generalize to optimal actions that do not appear in the training dataset. However, without function approximation, algorithms like DPO can mistakenly increase the likelihood of sub-optimal actions.
This establishes the importance of function approximation for the success of the algorithms such as DPO.   \looseness=-1
\end{itemize}

Taken together, our results establish the critical role \textit{coverage} plays in terms of convergence properties of preference learning algorithms as well as in the design of new, performant empirical approaches.

\neurips{\vspace{-3mm}}
\section{Related Work}
\neurips{\vspace{-3mm}}

\paragraph{Preference Fine-Tuning (PFT)} As discussed in the introduction of our work, 
there are two major paradigms for preference fine-tuning of LLMs. 
The first one, online RL methods \citep{ouyang2022training}, 
proposes to first train a reward model (classifier) to predict human preferences, 
followed by running an RL method to optimize this learned
reward function. While PPO \citep{schulman2017proximal} is the most popular RL algorithm used in the online RLHF framework by far \citep{stiennon2020learning,ouyang2022training, touvron2023llama},
 more recent work by \cite{ahmadian2024back} shows that simpler online RL algorithms like REINFORCE \citep{williams1992simple} also work well. The second class of methods, offline contrastive techniques \citep{zhao2023slic,rafailov2024direct,azar2024general}, avoid explicit reward modeling and directly optimize their objective on the offline preference dataset. Recently there are \textit{hybrid} methods that 
combine offline preference data with online preference labels \citep{xiong2023iterative,guo2024direct, rosset2024direct, azar2024general} -- we leave extending our analysis to this setting to future work. Throughout our paper, we assume for simplicity of analysis that preferences are generated by an underlying utility function and therefore contain no intransitivities \citep{munos2023nash,swamy2024minimaximalist}. 
\looseness=-1
\paragraph{Understanding PFT} 
Prior work has studied different parts of the standard RLHF recipe \citep{gao2023scaling,kirk2023understanding,singhal2023long,eisenstein2023helping} and the impact of preference data quality \citep{sharma2024critical}. In our work, we instead take a converge-based perspective on the relationship between online RL methods and offline contrastive methods. Although derived from the same minimum relative entropy objective \citep{ziebart2008maximum}
and perceived as equivalent by some early work \citep{rafailov2024direct,azar2024general}, more recent work has started to unravel the distinctions between these two classes of methods. \citet{tang2024understanding} repeatedly observe better performance from online rather than offline methods and after rigorously validating a variety of hypotheses, conclude that on-policy sampling is indispensable for ensuring a high quality policy. \citet{tajwar2024preference} perform an in-depth study of the effects of preference data, contrastive losses, and on-policy sampling and conclude that a combination of contrastive losses and interactive training is most preferable in practice. \citep{xu2024dpo} also observe better performance from online PPO than from offline DPO and argue this is because the former is able to eliminate a larger set of policies that are undesirable from the perspective of the later. We supplement these mostly empirical observations with a rigorous theoretical explanation for the observed behavior through the lens of dataset coverage, as well as designing an algorithm that addresses the key weaknesses of offline contrastive approaches.

Recent work \citep{yuan2024advancing,pal2024smaug,rafailov2024r} has observed an interesting effect of the DPO procedure: a simultaneous decrease in the likelihood of both preferred and rejected responses. This behavior is surprising at first glance because one would expect that DPO will increase the likelihood of preferred responses and decrease the likelihood of rejected responses. We provide a rigorous statistical explanation of this behavior and show that this behavior is natural when the offline preference data only contains sub-optimal responses but the function approximation allows DPO to extrapolate and generalize to the correct optimal responses. This highlights the role of function approximation in the success of offline contrastive based methods. 

\paragraph{Coverage}  We analyze online RLHF and offline contrastive-based methods via the concept of coverage.
Coverage measures how well an offline (data) distribution covers the support of the policy of interest, which has been the key technical tool in offline RL \citep{munos2008finite,xie2021bellman,uehara2021pessimistic,zhan2022offline}, offline-online RL \citep{ross2012agnostic, xie2021policy,song2022hybrid,amortila2024harnessing} and online RL \citep{kakade2002approximately,bagnell2003policy,xie2023the}. The data coverage plays an important role in our analysis since both online RLHF and offline contrastive-based methods rely on an offline preference dataset for learning.  \looseness=-1

\neurips{\vspace{-3mm}}
\section{Preliminaries}\label{sec:pre}
\neurips{\vspace{-2mm}}
Following a wide range of recent works \citep{rafailov2024direct,azar2024general},
we consider the RLHF problem in the contextual bandit formulation \citep{langford2008epoch}.
This is a reasonable simplification, as one can consider the generated sequence of 
tokens as one single action, due to the fact that the states are the generated tokens, 
and the dynamics are deterministic. We denote the context (prompt) space as $\Xcal$, 
and the action (response) space as $\Ycal$. Note that due to the finiteness of the 
possible tokens, the action space is finite but combinatorially large. 
We use $\rho \in \Delta(\Xcal)$ to denote the distribution of the prompts, and
$\pi: \Xcal \to \Delta(\Ycal)$ as policies (LLMs) that map prompts to a distribution 
of responses. We also consider the reward function class $\Rcal: \Xcal \times \Ycal \to \RR$,
which assigns a reward to each context-response pair.

We assume access to a reference policy $\piref$, which is usually referred to as the policy learned using supervised data when training the LLM, that needs to be further fine-tuned to align with human values. An offline preference dataset is collected in the format 
of $\Dcal = \{x,y^+, y^-\}$ triplets: given context $x \sim \rho$, the preference policy samples 
two responses $y^1,y^2 \sim \mu(\cdot \mid x)$, where $\mu$ is the offline response distribution. 
Previous works assume either $\mu$ to be the same distribution as 
$\piref$ \citep{rafailov2024direct} or different offline distribution \citep{azar2024general, rosset2024direct, gao2024rebel}.
Then, $y^1$ is labelled as $y^+$ 
(thus $y^2$ as $y^-$)
with probability $p^\ast(y^1 \succ y^2 \mid x)$, where $p^\ast$ is defined by the 
Bradley-Terry model \citep{bradley1952rank}:
\begin{align*}
  p^\ast(y^1 \succ y^2 \mid x) = \frac{\exp(r^\ast(x,y^1))}{\exp(r^\ast(x,y^1)) + \exp(r^\ast(x,y^2))},
\end{align*}

where $r^\ast$ is the human's implicit reward function. Note that this rules out intransitive preferences \citep{swamy2024minimaximalist, munos2023nash}. 
Throughout the paper, we will make the following assumption on the reward function:
\begin{assumption}[Boundedness of the reward]\label{assump:reward}
$\nrm*{r^\ast}_{\infty} \leq R.$
\end{assumption}

In many previous works, 
this formulation has been the 
canonical way to model the preference data in the RLHF literature 
\citep{christiano2017deep,rafailov2024direct,azar2024general}.
 The goal is to learn a policy $\pi$ to maximize the objective 
 $J(\pi)$, where 
\begin{align}\label{eq:objective}
  J(\pi)  = \EE_{x \sim \rho} \brk*{ \EE_{y \sim \pi(\cdot \mid x)} [r^\ast(x,y)] 
  - \beta \KL(\pi(\cdot \mid x) || \piref(\cdot \mid x))},
\end{align}
i.e., we want to both maximize the human implicit reward, and not deviate too much
from the reference policy. We denote the optimal policy $\pi^\ast \in \argmax_{\pi \in \Pi} J(\pi)$.  Here we call $\KL(\pi(\cdot \mid x) || \piref(\cdot \mid x))$ reverse KL because $\pi$ -- the policy to be optimized, appears first.  We will call $\KL(\piref(\cdot \mid x) || \pi(\cdot \mid x) )$ forward KL. By the definition of KL, we have \loose
\begin{align}
\text{Definition of reverse KL: }\qquad \KL(\pi(\cdot \mid x) || \piref(\cdot \mid x)) := \EE_{y\sim \pi(\cdot \mid x)} \brk*{ \ln (\pi(y|x) / \piref(y|x) )}. \label{eq:reverse_KL}
\end{align} Note that the expectation in reverse KL is under $\pi$, 
indicating that evaluating and optimizing reverse KL requires drawing \emph{online samples} from $\pi$. In contrast, evaluating forward KL only requires \emph{offline samples} drawn from $\piref$. As we will show, this key difference between reverse KL and forward KL plays an important role of separating online RLHF and offline contrastive methods such as DPO.
In this paper, we consider two types of algorithms:
online RL-based algorithms, and offline contrastive-based algorithms.

\neurips{\vspace{-2mm}}
\paragraph{Online RLHF Algorithms}
We consider algorithms such as 
\citet{christiano2017deep,ahmadian2024back} as the online RL based
methods. We abstract these algorithms as the following procedure: the algorithm
performs the following two-stage procedure: one first trains a reward model $\widehat{r}$
that minimizes the Bradley-Terry loss \footnote{We use $\widehat \EE$ to denote the empirical expectation over the dataset.}
\begin{align}\label{eq:bt_loss}
  \widehat{r} \in \argmax_{r \in \Rcal} \widehat \EE_{x,y^+, y^- \sim \Dcal} \brk*{ \log \prn*{\frac{\exp( r(x,y^+))}{\exp( r(x,y^+)) + \exp( r(x,y^-))}}},
\end{align}
and perform policy optimization (such as PPO \citep{schulman2017proximal})
to optimize the policy  with the reward model $\widehat{r}$:
\begin{align*}
  \pirlhf \in \argmax_{\pi} \widehat \EE_{x \sim \Dcal} \brk*{\EE_{y \sim \pi(\cdot \mid x)} [\widehat{r}(x,y)] - \beta \KL\prn*{\pi(\cdot \mid x) || \piref(\cdot \mid x)}}.
\end{align*}
However, this policy optimization step requires extensive online 
sampling, and possibly training an additional critic model (e.g., PPO), in addition to 
the reward model and policy. 
\neurips{\vspace{-2mm}}
\paragraph{Offline Contrastive Algorithms}
To circumvent the above-mentioned
computational burden, several purely offline contrastive-based methods
(i.e., without RL) have been proposed. In this paper, we focus 
on the following two most representative methods.
The first is Direct Preference Optimization (DPO) \citep{rafailov2024direct},
where the objective is 
$\pidpo \in \argmax_{\pi} \ldpo(\pi)$ with 
\begin{align}\label{eq:ldpo}
  \ldpo(\pi) =  \widehat \EE_{x,y^+, y^- \sim \Dcal} \brk*{ \log \prn*{ \frac{\exp \prn*{ \beta \log\prn*{\frac{\pi(y^+ \mid x)}{\piref(y^+ \mid x)}}}}{\exp \prn*{ \beta \log\prn*{\frac{\pi(y^+ \mid x)}{\piref(y^+ \mid x)}}} + \exp \prn*{ \beta \log\prn*{\frac{\pi(y^- \mid x)}{\piref(y^- \mid x)}}}}}}.
\end{align} 

Another offline contrastive method we will discuss in our paper is Identity Preference Optimization \citep{azar2024general}, but we will defer its technical details to the appendix.

\neurips{\vspace{-2mm}}
\section{Offline Contrastive Methods Require a Stronger Coverage 
Condition than Online RL Methods}
\neurips{\vspace{-2mm}}
We start by introducing the mathematical formulation of the coverage framework. 
The strongest coverage condition is the following global coverage condition \citep{munos2008finite}:
we say any offline distribution $\mu$ covers a policy $\pi$ if we have 
$\max_{x,y: \rho(x) > 0} \frac{\pi(y \mid x)}{\mu(y \mid x)} \leq \cglobal.$ Throughout this section, we will adopt the setting where $\mu = \piref$ \citep{rafailov2024direct}. 
Formally, we assume the following condition:
\begin{assumption}[Global Coverage]
  \label{assump:global}
  For all $\pi$, we have 
  \begin{align*}
    \max_{x,y: \rho(x) > 0} \frac{\pi(y \mid x)}{\piref(y \mid x)} \leq \cglobal.
  \end{align*}
  \neurips{\vspace{-3mm}}
\end{assumption} 
For the coverage terms, we always adopt the convention that $\frac{0}{0}= 0$. Note that one sufficient condition for this assumption is that, for any prompt 
$x$, and any token 
sequence $y$, we have $\piref(y \mid x) \geq 1/\cglobal$. 

As has been recognized in the offline RL literature, global coverage is a  
strong assumption, and efforts have been made to circumvent this assumption with 
more relaxed coverage conditions \citep{uehara2021pessimistic,chen2022offline,zhan2022offline}. 
In this paper, we will consider the following partial coverage assumption that is
weaker than \pref{assump:global}: \looseness=-1
\begin{assumption}[Local KL-ball Coverage]
  \label{assump:klball}
For all $\ekl <\infty$ and all policy $\pi$ such that $\EE_{x \sim \rho} \brk*{ \KL(\pi(\cdot \mid x) || \piref(\cdot \mid x))} \leq \ekl$,
  we have 
\neurips{\vspace{-3mm}}
  \begin{align*}
    \max_{x,y: \rho(x) > 0}\frac{\pi(y \mid x)}{\piref(y \mid x)} \leq C_{\ekl}.
  \end{align*}
\end{assumption} 
Note that $C_{\ekl}$ depends on $\ekl$. This coverage notion is relatively new in the RL literature and only appeared in previous analysis of RLHF algorithms
\citep{chang2024dataset}. We call this local coverage condition since it only requires $\piref$ to cover the policies that is within some KL-divergence ball centered at $\piref$. The intuition of this assumption is,
for any algorithm that can control the reverse KL of the output policy,
we can leverage the coverage condition to relate the error under the 
output policy to its error under the offline distribution, and thus
guarantee its performance. 
Finally, we note that since the policies with bounded KL is a subset 
of all policies, for a fixed $\piref$, we always have $C_{\ekl} \leq \cglobal$.

\begin{remark}
Taking a closer look at \pref{assump:klball}, we can see that this assumption is always true in the sense that for any policy with $\ekl < \infty$, $\max_{x,y: \rho(x) > 0}\frac{\pi(y \mid x)}{\piref(y \mid x)} < \infty$, i.e., $C_{\ekl} <\infty,$ for any $\ekl$. However, while being bounded, $C_{\ekl}$ can be large. Indeed a simple calculation can show that $\max_{x,y: \rho(x) > 0}\frac{\pi(y \mid x)}{\piref(y \mid x)}$ can be 
as large as $\max_{x,y: \pi(y \mid x) > 0} \exp\prn*{\frac{\ekl}{\pi(y \mid x)}}$. This can be undesirable because this suggests bounded reverse KL itself 
is not enough to guarantee optimality: the error can have an 
\emph{exponential} amplification when switching from $\piref$ to $\pi$. 
Thus this motivates \pref{assump:klball}, which assumes that $C_{\ekl}$
is reasonably small.
\end{remark}

In what follows, we will show that the global coverage assumption 
(\pref{assump:global}) is necessary for
offline contrastive-based algorithms such as DPO and IPO, 
but partial coverage assumption such as \pref{assump:klball}
is sufficient for online RL based algorithms. This establishes a
separation between the two types of algorithms. We emphasize this 
theoretical separation explains why in practice online methods 
is less prone to problems such as reward hacking and producing out-of-distribution responses that are due to dataset with insufficient coverage.  

\subsection{Global Coverage is Necessary for Offline Contrastive Algorithms}\label{sec:necessary}

\paragraph{Failure of DPO Under Partial Coverage}
Now we show that if the strong coverage \pref{assump:global} breaks, then DPO can not 
guarantee any performance with respect to the objective function \pref{eq:objective}. 
The intuition is based on a rather common observation of the DPO algorithm: the 
DPO policy $\pidpo$ may generate out of distribution responses, while in contrast, RLHF does not generate responses outside of the support of $\piref$ due to online reverse-KL constraint. For example, \citep{xu2024dpo} provides a
construction where $\pidpo$ chooses a response where RLHF policy assigns 0 mass, thus proving that RLHF policies are a subset 
of DPO policies.

However, such construction assumes that the reward learning
 procedure 
of DPO makes arbitrarily large errors. Also, previous constructions assume 
deterministic preference, which is only true if the underlying reward function is
unbounded. This violates the natural assumption of \pref{assump:reward}.
In the following, we relax these constraints and thus
show that DPO fails to guarantee any performance in a rather strong sense.
Concretely, DPO constructs the following implicit reward class with the policy class $\Pi$:
$\Rdpo = \crl*{\beta \log \prn*{\frac{\pi(y \mid x)}{\piref(y \mid x) Z(x)}} \mid \pi \in \Pi}$, where $Z(x)$ is a partition function that maps context to a real number and is independent of $y$. Plugging this formulation into the BT loss (\pref{eq:bt_loss}) recovers exactly the DPO loss (\pref{eq:ldpo}) as the partition functions are canceled. Now we can characterize the returned policy by DPO as exactly whose corresponding reward function is accurate \emph{in distribution}: \looseness=-1

\begin{assumption}[In Distribution Reward Learning]
\label{assump:reward_learning}
We assume the DPO policy $\pidpo$ satisfies that:
\begin{align*}
 \EE_{x,y \sim \rho \circ \piref} \brk*{ \prn*{\beta \log \prn*{\frac{\pidpo(y \mid x)}{\piref(y \mid x) Z(x)}}-r^\ast(x,y)}^2} \leq \edpo.
\end{align*}
\end{assumption}

Note that this is a rather strong assumption for BT loss -- by \pref{lem:reward_relative}, at best one can only hope: for any learned reward function $\widehat r$,
for each context $x$, there exists a constant $c(x)$ such that \looseness=-1
\begin{align}\label{eq:reward_gap}
  \EE_{x,y \sim \rho \circ \piref} \brk*{ \prn*{\widehat r(x,y) - r^\ast(x,y) - c(x)}^2} \leq \varepsilon,
\end{align}
i.e., the reward model predicts the human reward up to a gap that is independent of $y$.
This is due to the fact that BT loss only requires the reward function to capture the 
relative difference, or in other word, any constant shift (with respect to context) in the reward will be canceled
in the BT loss. 
However, for the rest of the section, we will make the stronger learning assumption
that the gap $c(x) = 0$ (such as in the case of \pref{assump:reward_learning}). 
Previous counterexamples analysis 
violates this assumption, but we will show that even under this strong assumption,
DPO still can not guarantee any performance. 

\begin{proposition}\label{prop:dpo_partial}
Denote $\piref$ as \emph{any} reference policy such that \pref{assump:global} breaks. 
Let $\Pi_{\textsf{dpo}}$ be the set of DPO returned policies such that
\pref{assump:reward_learning} holds. Then there exists policy $\pi \in \Pi_{\textsf{dpo}}$
such that $J(\pi) = -\infty$.
\end{proposition}

\begin{proof}[Proof sketch]
   Without loss of generality, we consider a promptless setting, and assume that
  the response space is $\Ycal = \{y_1,y_2,y_3\}$. Again without loss of generality,
  we assume $\piref$ only covers $y_1$ and $y_2$, and thus \pref{assump:global} breaks. 
  We assume partition function $Z = 1$ for all $\pi$ 
  but we will be rigorous in the formal proof.
  Then consider the following policy $\pi$ such that
  \begin{align*}
    \beta \log \prn*{\frac{\pi(y_1)}{\piref(y_1)}} = r^\ast(y_1) - \sqrt{\edpo}, \mathand \beta \log\prn*{ \frac{\pi(y_2)}{\piref(y_2)} } = r^\ast(y_2) - \sqrt{\edpo},
  \end{align*}
  One can check $\pi$ satisfies \pref{assump:reward_learning}.
  Now consider the optimal policy $\pi^\ast(y_i) = \piref(y_i) \exp\prn*{\frac{1}{\beta}r^\ast(y_i)}$, 
  for $i \in \{1,2\}$, and $\pi^\ast(y_3) = 0$. Since $\pi^\ast(y_1) + \pi^\ast(y_2) = 1$,
  combining everything we get $\pi(y_3) > 0$, which implies $\KL\prn*{\pi || \piref}$ is unbounded, 
  thus we complete the proof. 
\end{proof}

One can first relate the above construction to the parital coverage assumption \pref{assump:klball}:
since the policy $\pi$ considered in the proof has unbounded reverse KL with respect to $\piref$,
thus it is not in the KL-ball of $\ekl$ around $\piref$, which implies that \pref{assump:klball}
is not sufficient for DPO. Next we show that global coverage is necessary for the IPO algorithm.

\paragraph{Failure of IPO Under Partial Coverage} 
To show that the global coverage is necessary for IPO, we
can even assume a stronger in-distribution learning guarantee,
that is, the returned policy achieves the smallest
error on its population loss 
in distribution. 
\begin{proposition}[Informal]\label{prop:ipo_partial}
Denote $\piref$ as \emph{any} reference policy such that \pref{assump:global} breaks.
Let $\Pi_{\textsf{ipo}}$ be the set of IPO returned policies such that it is the minimizer of 
in-distribution error on its population loss.
Then there exists policy $\pi \in \Pi_{\textsf{ipo}}$ such that $J(\pi) = -\infty$.
\end{proposition}

We defer the detailed setup and formal version to \pref{app:missing_proofs}, but the 
construction for the above proofs share the same intuition: 
the reverse KL term in the objective function
can be unbounded. 
For offline contrastive-based algorithms, the KL regularization is only enforced 
under the data distribution, and thus the algorithm can not guarantee bounded reverse KL
if the reference policy does not cover the response space well.
Although we only showed counterexamples for DPO and IPO, we conjecture that 
the same intuition holds for other offline contrastive-based algorithms.
One natural question at this point would be: how about the forward 
KL? Not surprisingly, the forward KL for DPO (but we conjecture for
other offline constructive-based methods as well) is vacuously large, 
and we formalize this result in \pref{sec:dpo_forward}.

\begin{remark}\label{remark:pointwise} The folklore that DPO is equivalent to RLHF is often based on some assumption that is much stronger than \pref{assump:reward_learning}: it requires that the learned policy has a point-wise accuracy guarantee  $\beta \ln ( \pidpo(y|x) / \piref(y|x)) = r^\ast(x,y)$ for all $x,y$. Such a point-wise guarantee is unrealistic in reality and does not hold in general in the supervised learning sense. The in-distribution style guarantee in \pref{assump:reward_learning} is the best one could hope for from a supervised learning algorithm.
\end{remark}

\neurips{\vspace{-2mm}}
\subsection{Global Coverage is Sufficient for Offline Contrastive Algorithms}
\neurips{\vspace{-1mm}}
After showing that global coverage is necessary for DPO to guarantee any performance,
we now show that it is sufficient for the performance guarantee. 
\begin{theorem} \label{thm:dpo_global}
  Let $\piref$ be any reference policy such that \pref{assump:global} holds. For any policy $\pidpo$ such that the event in \pref{assump:reward_learning}
  holds, we have that 
  \begin{align*}
    J(\pi^\ast) - J(\pidpo) = O(\cglobal \sqrt{\edpo}).
  \end{align*}
\end{theorem}
\begin{proof}
By \pref{lem:objective_decomposition}, we have 
\neurips{
{
\begin{align*}
    J(\pi^\ast) - J(\pidpo) &\leq 
    \EE_{x \sim \rho} \EE_{y^1 \sim \pi^\ast(\cdot \mid x), y^2 \sim \pidpo(\cdot \mid x)} 
        \brk*{r^\ast(x,y^1) - \rdpo (x,y^1) - r^\ast(x,y^2) + \rdpo (x,y^2)} \\
    &\leq \sqrt{\EE_{x \sim \rho} \EE_{y^1 \sim \pi^\ast(\cdot \mid x), y^2 \sim \pidpo(\cdot \mid x)} 
    \brk*{ \prn*{r^\ast(x,y^1) - \rdpo(x,y^1) - r^\ast(x,y^2) + \rdpo(x,y^2)}^2}} \\
    &\leq \sqrt{\cglobal^2 \EE_{x \sim \rho} \EE_{y^1, y^2 \sim \piref(\cdot \mid x)} 
    \brk*{ \prn*{r^\ast(x,y^1) - \rdpo(x,y^1) - r^\ast(x,y^2) + \rdpo(x,y^2)}^2}},
\end{align*}
}
}
\arxiv{
\begin{align*}
    J(\pi^\ast) - J(\pidpo) &\leq 
    \EE_{x \sim \rho} \brk*{ \EE_{y^1 \sim \pi^\ast(\cdot \mid x), y^2 \sim \pidpo(\cdot \mid x)} 
        \brk*{r^\ast(x,y^1) - \rdpo (x,y^1) - r^\ast(x,y^2) + \rdpo (x,y^2)} } \\
    &\leq \sqrt{\EE_{x \sim \rho} \brk*{ \EE_{y^1 \sim \pi^\ast(\cdot \mid x), y^2 \sim \pidpo(\cdot \mid x)} 
    \brk*{ \prn*{r^\ast(x,y^1) - \rdpo(x,y^1) - r^\ast(x,y^2) + \rdpo(x,y^2)}^2}}} \\
    &\leq \sqrt{\cglobal^2 \EE_{x \sim \rho} \brk*{ \EE_{y^1, y^2 \sim \piref(\cdot \mid x)} 
    \brk*{ \prn*{r^\ast(x,y^1) - \rdpo(x,y^1) - r^\ast(x,y^2) + \rdpo(x,y^2)}^2}}},
\end{align*}
}
and we can complete the proof by plugging in the error guarantee from \pref{assump:reward_learning}.
\end{proof}

Note that as the proof suggests, the result holds with the more general reward learning guarantee as in 
\pref{lem:reward_relative} -- one only needs to be accurate in predicting the 
relative rewards between response pairs.

\subsection{Online RL Method Under Partial Coverage}
Finally, we contrast the previous negative results in \pref{sec:necessary} for offline contrastive-based algorithms
to a positive result for online RL-based algorithms, under the partial coverage 
setting. We will show that in general global coverage is not necessary for RLHF, i.e., it 
can guarantee performance under partial coverage. 
In fact, one might still be able to show an impossibility result for RLHF under partial coverage,
by reusing the same counterexample as in the previous section (c.r., \pref{prop:dpo_partial}). 
Concretely, as long as the learned reward $\widehat r (y_3) \to \infty$, $\pirlhf(y_3)$ will be $1$
and thus the reverse KL will be unbounded. However, this is a rather unrealistic scenario,
as the construction requires a reward model (e.g., a neural network) to output an unbounded value.
Thus this motivates the following assumption:
\begin{assumption}\label{assump:reward_bounded}
  For all learned reward model $\widehat r$ from the reward model class, we have that $\nrm*{\widehat r}_{\infty} \leq R'$.
\end{assumption} 
At this point, one might ask why a similar assumption is missing for the offline contrastive-based
analysis, since in \pref{remark:pointwise} we argued that a point-wise learning guarantee is unrealistic but \pref{assump:reward_bounded} is indeed also a point-wise boundedness assumption. The reason lies in the different construction of the model class $\widehat r$ for those algorithms:
for DPO and IPO, the reward model is constructed as $\rdpo = \beta \log \prn*{\frac{\pi}{\piref \cdot Z}}$,
and there is no natural function class for $\pi$ such that point-wise assumptions such as the one in \pref{remark:pointwise} or \pref{assump:reward_bounded} holds. 
In contrast, post-processing such as clipping, offline normalization and on-the-fly normalization of rewards is standard in practice, which means the policy will always witness bounded rewards \citep{ chang2023learning,chang2024dataset,gao2024rebel,  ahmadian2024back} during online RL training (e.g., PPO).  As we will show in the following, 
the difference in the reward function (which is tied to the offline vs. online 
nature of the algorithms) can explain the different coverage requirement 
of the algorithms. Note that we use the same in-distribution reward learning assumption for both types of methods. 

To relate to \pref{assump:klball}, we first show that the reverse KL divergence of the RLHF policy
is always bounded under \pref{assump:reward_bounded}.

\begin{lemma}\label{lem:rlhf_reverse_kl}
Suppose that \pref{assump:reward_bounded} holds. Then for any RLHF policy $\pirlhf$, we have that
\begin{align*}
  \KL(\pirlhf || \piref) := 
  \EE_{x \sim \rho} \brk*{ \EE_{y \sim \pirlhf(\cdot \mid x)} \brk*{\log \prn*{\frac{\pirlhf(y \mid x)}{\piref(y \mid x)}}}}
  \leq \frac{2R'}{\beta}.
\end{align*}
\end{lemma}
Then we can show that the RLHF algorithm can guarantee performance under partial coverage:
\begin{theorem}\label{thm:rlhf_partial}
  Suppose that \pref{assump:reward_bounded} holds. Then for any reference policy $\piref$ for which 
  \pref{assump:klball} holds with $\ekl = \frac{2R'}{\beta}$, and any RLHF policy $\pirlhf$
  with $\widehat r$ such that (c.r. \pref{assump:reward_learning})
  \neurips{$\EE_{x,y \sim \rho \circ \piref} \brk*{ \prn*{ r^\ast(x,y) - \widehat r(x,y)}^2 }\leq \ereward,$}
  \arxiv{
  \begin{align*}
  \EE_{x,y \sim \rho \circ \piref} \brk*{ \prn*{ r^\ast(x,y) - \widehat r (x,y)}^2 }\leq \ereward,
  \end{align*} 
  }
    we have 
  \begin{align*}
    J(\pi^\ast) - J(\pirlhf) \leq O(C_{\ekl} \sqrt{\ereward}).
  \end{align*}
\end{theorem}

Conditioned on \pref{lem:rlhf_reverse_kl}, the proof of this theorem is similar to that of \pref{thm:dpo_global} so we 
defer it to \pref{app:missing_proofs}. Similar to \pref{thm:dpo_global}, we note that 
\pref{thm:rlhf_partial} holds under a weaker reward learning guarantee as in \pref{lem:reward_relative}. 
We also remark that as long as $\ekl$ is finite, $C_{\ekl}$ is finite,
so the bound is never vacuous. \emph{Since $C_{\ekl} \leq \cglobal$ for all $\ekl$, it indicates the regret bound of RLHF is never worse and can be much better than the regret bound of DPO.}
Combining \pref{thm:dpo_global} and \pref{thm:rlhf_partial}, we complete the separation result between offline contrastive methods and online RL methods. \looseness=-1 

A natural question at this point could be: can we further relax the local KL-ball coverage condition in \pref{assump:klball} to a single-policy coverage condition, i.e., just assuming $\max_{x,y} \pi^*(y|x) / \piref(y|x) \leq C$? Prior work \cite{zhan2023provable} shows that with explicit pessimism, it is possible. However, using pessimism makes the algorithm from \cite{zhan2023provable} not computationally tractable and hard to scale to LLM experiments. Our conjecture is that for the RLHF policy $\pirlhf$, it is not possible to achieve meaningful regret under the single policy coverage condition, due to KL not being strong enough to induce pessimism (i.e., bounded KL between $\pi$ and $\piref$ can still imply exponentially large density ratio $\pi / \piref$). Developing a lower bound for $\pirlhf$ under single policy coverage in this case can be an interesting future work.

\section{Hybrid Preference Optimization: Regularizing Offline Learning with Unlabeled Online Samples} \label{sec:hypo}

\begin{algorithm}[t]
  \caption{Hybrid Preference Optimization (HyPO)}
  \begin{algorithmic}[1]
      \Require Pretrained LLM $\pi_{\theta_0}$, 
      reference policy $\piref$,
      offline data $\Dcal$,
      learning rate $\alpha$,
      KL coefficient $\lambda$.
      \For{$t = 1, \dots, T$}
          \State Sample a minibatch of \textcolor{red}{offline} data $D_{\mathsf{off}} := \{x, y^+, y^-\} \sim \Dcal$.
          \State Compute DPO loss $\ell_{\mathsf{dpo}} := \sum_{x,y^+,y^- \in D_{\mathsf{off}}} \log \prn*{\sigma\prn*{\beta \log \prn*{\frac{\pi_{\theta_{t-1}}(y^+ \mid x)}{\piref(y^+ \mid x)}} - \beta \log \prn*{\frac{\pi_{\theta_{t-1}}(y^- \mid x)}{\piref(y^- \mid x)}}} } $.
          \State Sample (unlabeled) \textcolor{red}{online} data $D_{\mathsf{on}} := \{x, y\}$ where $x\sim \Dcal, y \sim \pi_{\theta_{t-1}}(x)$.
          \State Compute $\ell_{\mathsf{kl}} := \sum_{x,y\in D_{\mathsf{on}}} \log(\pi_{\theta_{t-1}}(y|x)) \cdot \mathrm{sg}\prn*{\log \prn*{\frac{(\pi_{\theta_{t-1}}(y|x))}{(\piref(y|x))}}}$.
          \State Update $\theta_t = \theta_{t-1} + \alpha \cdot \nabla_{\theta_{t-1}} \prn*{\ell_{\mathsf{dpo}} - \lambda \ell_{\mathsf{kl}}}$.
      \EndFor
  \Return{$\pi_{T}$.}
  \end{algorithmic}
  \label{alg:hypo}
\end{algorithm}

In this section, we will provide a practical algorithm that bridges the gap between
the offline contrastive-based algorithms and the online RL-based algorithms.
As we see in the previous sections, the difference between the two types of algorithms
is their reward model parametrization, and whether to perform online rollouts. In the following, we will show that these two properties are in fact tightly intervened with each other. 

Here we will focus on the DPO algorithm. One way to fix the issue of the unbounded reverse KL of DPO (which is caused by the 
unbounded reward model class)
is to consider the following ideal procedure: at the beginning of the 
algorithm, we first go through the policy class $\Pi$, and then we filter out all 
the policies such that $\KL(\pi || \piref) \geq \frac{2R'}{\beta}$, where $R'$ is the
boundedness of the reward function class for RLHF. Now applying the same analysis 
of \pref{thm:rlhf_partial}, we can show that this revised DPO algorithm can guarantee
performance under the partial coverage assumption, because 
now the \pref{lem:rlhf_reverse_kl}, a sufficient condition for \pref{thm:rlhf_partial}, is explicitly enforced by the 
constraints. 
We defer the detailed statement 
and analysis to \pref{sec:hypo_theory}. 

However, such a filtering procedure is not possible in practice, but we can instead 
consider the following constrained optimization problem: we call the 
definition of DPO loss in \pref{eq:ldpo}, we want to solve
\neurips{\vspace{-2mm}}
\begin{align}\label{eq:constrained_dpo}
  \max_{\pi} \ldpo(\pi) \quad \text{s.t.} \quad \KL(\pi || \piref) \leq \frac{2R'}{\beta},
\end{align}
using the KKT conditions, we can show that the following Lagrangian form is equivalent to
\pref{eq:constrained_dpo}:
\begin{align}\label{eq:lagrange_dpo}
  \max_{\pi} \ldpo(\pi) - \lambda \KL(\pi || \piref),
\end{align}
where $\lambda$ is the Lagrange multiplier. However, in reality, since we do not 
know the exact value of $R'$, we can consider setting $\lambda$ to be a hyperparameter. We present the pseudocode in \pref{alg:hypo}.
Note that due to the reverse KL term, the Hybrid Preference Optimization (HyPO) algorithm optimizes \pref{eq:lagrange_dpo} via both offline and online samples where the offline samples are used for constructing and optimizing $\ldpo$ (here $\sigma$ denotes the sigmoid function), and the online samples $y\sim \pi(\cdot \mid x)$ are for $\KL$ (i.e., $\ell_{kl}$). Note that regularizing with reverse KL via online samples is widely used in online RLHF (e.g., PPO \citep{stiennon2020learning}, APA \citep{zhu2023fine}, REBEL \citep{gao2024rebel}).
Here $\mathrm{sg}$ refers to the stop gradient operation, 
which is a common practice in optimizing reverse KL in the LLM fine-tuning 
setting \citep{ouyang2022training, vonwerra2022trl}.
\arxiv{We also remark that a few recent and concurrent works \citep{xu2024contrastive,fisch2024robust,liu2024provably} 
propose to regulate DPO with an additional \emph{forward KL} term, where 
the samples are directly from the reference policy, while 
HyPO uses reverse KL which requires sampling from the learned 
policy online. 
}Finally, previous iterative RLHF methods \citep{xiong2024iterative} can be interpreted as hybrid methods as well, but they require labeling online samples from an additional reward model while HyPO only requires unlabeled online samples.

\begin{table}
  \centering
  \caption{Results on TL;DR dataset. Winrate is evaluated by GPT4 and RM score is from the trained reward model. Experiments are repeated for 3 random seeds. Mean and standard deviation are reported. \looseness=-1}
  \begin{tabular}{ccccc}
    \toprule
    Model size & Algorithm & Winrate $(\uparrow)$ & RM score $(\uparrow)$ & $\KL(\pi || \piref) (\downarrow)$ \\
    \midrule
    \multirow{2}{*}{1.4B} &
    DPO & 42.17\% (2.5\%) & 0.16 (0.05) & 44.90 (1.29) \\
     &
    HyPO & \textbf{46.44\% (2.39\%)} & \textbf{0.37 (0.05)} & \textbf{27.07 (2.34)} \\
    \midrule
    \multirow{2}{*}{2.8B} &
    DPO & 44.39\% (0.4\%) & 2.43 (0.10) & 68.95 (3.08) \\
     &
    HyPO & \textbf{50.50\% (1.89\%)} & \textbf{2.51 (0.13)} & \textbf{48.98 (4.23)} \\
    \bottomrule
  \end{tabular}\label{table:results}
\end{table}

\begin{table}
\centering
\caption{Results on general chat benchmarks. We evaluate the base model (Meta-Llama-3-8B-Instruct), DPO-fine-tuned model, and HyPO-fine-tuned model.}
\begin{tabular}{c|ccc|cc}
\toprule
\multirow{2}{*}{Model} & \multicolumn{3}{c|}{MT-Bench} & \multicolumn{2}{c}{AlpacaEval 2.0} \\
 & 1st Turn & 2nd Turn & Average & LC Win Rate & Win Rate \\
\midrule
Meta-Llama-3-8B-Instruct \citep{meta2024introducing} & 8.31 & 7.89 & 8.10 & 26.0 & 25.3 \\
DPO-Llama-3 & 8.08 & 7.41 & 7.75 & 28.4 & 30.9 \\
HyPO-Llama-3 & 8.43 & 7.75 & 8.09 & 30.7 & 32.2 \\
\bottomrule
\end{tabular}\label{table:results_chat}
\end{table}

\begin{table}[t]
\caption{Results on Open LLM leaderboard. We evaluate the base model (Meta-Llama-3-8B-Instruct), DPO-fine-tuned model, and HyPO-fine-tuned model.}
\centering
\begin{tabular}{c|ccccc|c}
\toprule
Model & 
\makecell{MMLU\\(5-shot)} & 
\makecell{GSM8K\\(5-shot)} & 
\makecell{Arc\\(25-shot)} & 
\makecell{TruthfulQA\\(0-shot)} & 
\makecell{HellaSwag\\(10-shot)} & 
Average \\
\midrule
\makecell{Meta-Llama-3-8B-Instruct \\ \citep{meta2024introducing}}& 65.68 & 74.91 & 62.12  & 43.88 & 78.76 & 65.07 \\
DPO-Llama-3 & 65.82 & 73.62 & 63.14  & 45.02 & 79.1 & 65.34 \\
HyPO-Llama-3 & 65.74 & 73.84 & 62.71  & 45.55 & 79.74 & 65.51 \\
\bottomrule
\end{tabular}\label{table:results_nlp}
\end{table}

\arxiv{\subsection{Experimental Results}}
\neurips{\paragraph{Experimental Results}}

\subsubsection{Summarization}
Our first experiment is on the TL;DR dataset \citep{stiennon2020learning}. Our experiment 
setup mostly follows \citep{gao2024rebel}:  we use a maximum context
length of 512 and a maximum generation length of 53. We use Pythia 1.4B and Pythia 2.8B \citep{biderman2023pythia}
as the pre-trained model. For the supervised fine-tuning
(SFT) model, we train it over 1 epoch of the dataset with 
human reference responses as labels. We train the reward model on top of the SFT over 1 epoch of preference data. Both HyPO and DPO are trained over 1 epoch of preference data with Low-rank Adaptation (LoRA)
\citep{hu2021lora}. We defer more experiment details in \pref{sec:exp_details}. \looseness=-1

We summarize the results in \pref{table:results}: HyPO outperforms 
DPO in terms of GPT4 win-rage and reverse KL. Particularly, the significant reduction in reverse KL implies the impact of including a reverse KL term explicitly into the DPO objective.  While comparing with PPO (e.g., Table 1 in \citet{gao2024rebel}), HyPO's performance is still lower in winrate, HyPO does preserve the key advantages of DPO over PPO: we avoid training additional reward model and a value network. \loose

\subsubsection{General Chat}
In the general chat setting, the model is required to produce a response $y$ given user instruction $x$. We again follow the experiment setup in \citet{gao2024rebel}, where we finetune the Meta-Llama-3-8B-Instruct \citep{meta2024introducing} model on the ultrafeedback dataset \citep{cui2023ultrafeedback}. Due to the computation constrain, we follow the setup in \citet{gao2024rebel} where we only train the last 4 layers of the network for both HyPO and DPO.

For evaluation, we use the common metrics including AlpacaEval 2.0 \citep{dubois2024length}, MT-bench \citep{zheng2024judging} and Open LLM leaderboard tasks: MMLU \citep{hendrycks2020measuring}, GSM8K \citep{cobbe2021training}, Arc \citep{clark2018think}, TruthfulQA \citep{lin2021truthfulqa} and HellaSwag \citep{zellers2019hellaswag}. We provide the results for AlpacaEval and MT-bench in \pref{table:results_chat}, and the the results of the remaining tasks can be found in \pref{table:results_nlp}.

\subsubsection{HyPO Mitigates Overfitting in Contrastive Methods}
Since the offline contrastive based methods only work with a static offline dataset, the overfitting issue has been observed \citep{tang2024understanding}. In our last experiment, we show that HyPO can effectively address the overfitting issue by leveraging the unlabeled online data. We follow the setup of the summarization task with Pythia-2.8B base model. We train DPO and HyPO for 5 epochs respectively, and evaluate on the first 300 data in the validation dataset. We plot the validation KL in \pref{fig:overfitting}: we observe that HyPO is better at preventing the deviation from the reference policy caused by overfitting from training on excessive epochs, even though the methods theoretically both have KL regularization to the reference policy. 

\begin{figure}[t]
    \centering
    \includegraphics[width=0.3\linewidth]{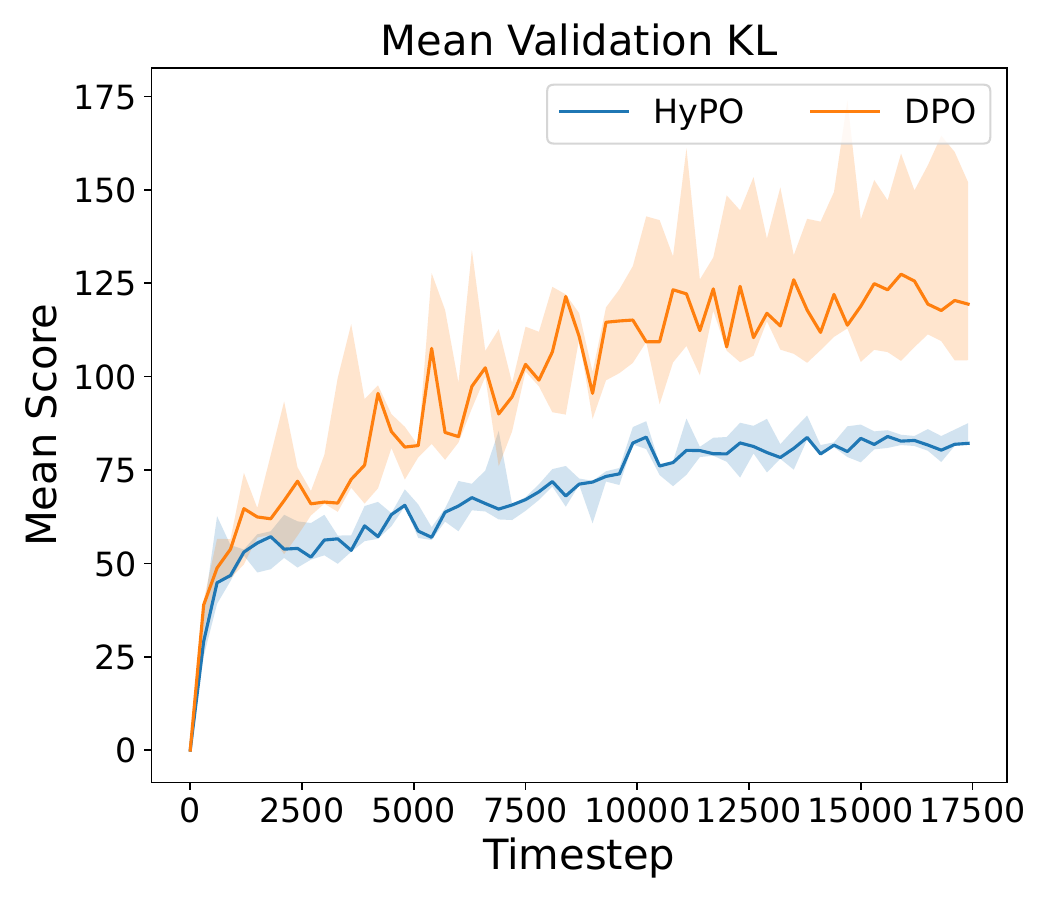}
    \caption{Mean validation reverse KL to the reference policy when DPO and HyPO are trained for 5 epoch on the TL;DR dataset. We repeat the experiment for 3 random seeds and plot the median and the shaded areas denote the min and max over the 3 repetitions.}
    \label{fig:overfitting}
\end{figure}

\neurips{\vspace{-2mm}}
\section{Function Approximation Coverage: Can
Fine-tuned Policies Extrapolate?}
\neurips{\vspace{-2mm}}

Our final result is a theoretical explanation of the extrapolation 
behavior of preference fine-tuning algorithms under the global coverage assumption in the function 
approximation (FA) setting. 
The extrapolation behavior refers to the phenomenon of RLHF algorithms (e.g., DPO) can improve SFT models despite the fact that during training the policies 
assign decreasing likelihood to both the preferred  and rejected responses
 (i.e., they must increase the likelihood of responses outside of the training data) \citep{pal2024smaug}.

A previous attempt \citep{rafailov2024r} to explain this behavior is based on the assumption that the responses from the reference policy have the same 
distribution as the \emph{preferred} responses from the dataset, i.e., 
$y^+ \sim \mu \stackrel{d}{=} y \sim \piref$. However, as mentioned in \pref{sec:pre}, more realistically, one should assume that 
$y \sim \mu \stackrel{d}{=} y \sim \piref$ since it is more natural to use the reference policy to generate pairs of responses to collect labels; or even more 
generally by considering 
$\supp(\Dcal) \subset \supp(\piref)$. The latter is common in practice, for example, 
the dataset is often precollected, or the reference policy might have a small mass on 
some responses, so with a high probability they are not sampled during the data collection process. 

In the following example, we illustrate this behavior using linear function approximation.  
We  use an offline dataset that does not contain the optimal action. We show that thanks to the linear function approximation and the dataset coverage, DPO has hope to extrapolate correctly, i.e.,  it can increase the model's likelihood of the optimal action while decreasing the likelihood of both the preferred and rejected actions from the offline data. 

\begin{example}
Consider a promptless setting, where the response space is $\Ycal = \{y_1,y_2,y_3\}$.
Consider the linear function approximation setting with feature map $\phi$, where $\phi(y_1) = [1,0],
\phi(y_2) = [1/2,1/2], \phi(y_3) = [0,1]$. Suppose all policies are parametrized as softmax linear policies, i.e., $\pi(y) \propto \exp(w^\top_\pi \phi(y))$. Let $w_{\mathsf{ref}} = [1,1]$, then we have $\piref(y_i) = 1/3, \forall i \in \{1,2,3\}$.

Consider the ground truth reward function $r^\ast(y) = [10,1]^{\top}\phi(y)$, 
and suppose $\supp(\mu) = \crl*{y_1,y_2}$, i.e., the data only covers $y_1$ and $y_2$. And as always, the preference is based on 
the ground truth reward function under the Bradley-Terry model.

We can first check that the data distribution indeed has global coverage in the 
linear function approximation case \citep{xiong2022nearly}, i.e., let $\Sigma_{\mu} = \EE_{y \sim \mu} \phi(y)\phi(y)^{\top}$,
then for all $\pi$, 
\begin{align*}
  \EE_{y \sim \pi} \|\phi(y)\|^2_{\Sigma_{\mu}^{-1}} \leq C_{\pi}.
\end{align*} 
If we parameterize
$\widehat  r (y) = \widehat {w}^\top \phi(y)$ (or in case of DPO, we can still check and see that $\rdpo (y) = \widehat {w_{\textsf{dpo}}}^\top \phi(y)$ because of the softmax linear parametrization of the policies), for either direct reward learning or DPO, we can have the learned reward function $\widehat  r(y) = [10,1]^{\top} \phi(y) + c$, where $c$ is the constant reward shift (c.r. \pref{eq:reward_gap}). Then a simple calculation 
(by $\pi(y) \propto \piref(y) \exp (\widehat  r(y)/\beta)$)
shows that, as long as $c$ is small enough, the policies will decrease the likelihood of
$y_1$ and $y_2$ and increase the likelihood of $y_3$.
\end{example} 

The above example shows when the training dataset together with the function approximation allow the learned function to generalize (e.g., learn a function that can predict well on test examples beyond the training data --- a property supervised learning can have), algorithms like DPO can extrapolate correctly, i.e., they can push up the likelihood of the optimal responses outside of the training data while pushing down the likelihood of all the responses in the training data.   

\neurips{
To validate our theory result, in \pref{sec:extrapolation_exp} we perform a synthetic experiment on
global coverage with linear function approximation. As shown in \pref{fig:extrapolation}, this extrapolation behavior is observed in 
both online RL method and DPO. 

To further demonstrate the importance of function approximation and generalization, we conduct the same experiments but without the linear function approximation (i.e., treat each action as independent just like one would do in the classic multi-armed bandit setting). In this case, trying one action does not give us information about the other action. 
We see that in this case, DPO can erroneously
assign a higher probability to unseen suboptimal responses instead of the unseen optimal response, indicating DPO can fail to extrapolate and generalize correctly. 
Our investigation identifies that function approximation and generalization play an important role in the success of RLHF and DPO algorithms.
}

\arxiv{
\subsection{Synthetic Experiment for Extrapolation} \label{sec:extrapolation_exp}

To validate our theory result, in this section we perform a synthetic experiment on
global coverage with linear function approximation. As shown in \pref{fig:extrapolation}, this extrapolation behavior is observed in 
both online RL method and DPO. 

\subsubsection{Extrapolation with Function Approximation}

We first describe our experiment setup. We consider linear function 
approximation setting where we have 100 responses ($|\Ycal| = 100$).
We consider a 16-dimensional feature vector $\phi: \Ycal \to \RR^{16}$,
and we generate $\phi(y)$ by simply sampling 99 random 16-dimensional 
vectors where the $\ell_1$ norm of each vector is 1. We add one final 
$\phi(y) = [1,0,0,\dots]$. We construct the implicit human reward $r^\ast (y) = {w^\ast}^{\top} \phi(y)$, where $w^\ast = [5,...]$, and the rest of the entries are 
sampled from Unif(-2,2). We parametrize the policies as softmax linear policies, i.e., 
we parametrize each policy $\pi$ with $w^\pi \in \RR^{16}$
such that $\pi(y) = \frac{{w^\pi}^\top \phi(y)}{\sum_{y \in \Ycal}
{w^\pi}^\top \phi(y)}$. One can check in this formulation the 
implicit reward in DPO ($\rdpo$) is linear in $\phi$. We generate 10000 preference pairs, according to the BT model 
under $r^\ast$, for the first 50 responses. We checked that the 
first responses indeed span $\RR^{16}$. Thus the offline data has 
global coverage in linear function approximation setting. 

For on-policy RL methods, we first train a reward model. Then we 
simply perform gradient descent on the KL-regularized bandit loss 
(we assume $\piref$ is uniform). For DPO, we simply perform SGD 
on the offline preference dataset. We track two qualities over 
the training: the mean log probability of a random subset of preferred 
responses, and the log probability of best response $\phi(y) = [1,0,0,\dots]$. We plot the results in \pref{fig:extrapolation} (Left and middle).
We observe that both methods have the extrapolation behavior --
the probability of preferred responses decays but the probability of 
the optimal response goes up. \loose

\begin{figure}
    \centering
    \includegraphics[width=0.3\linewidth]{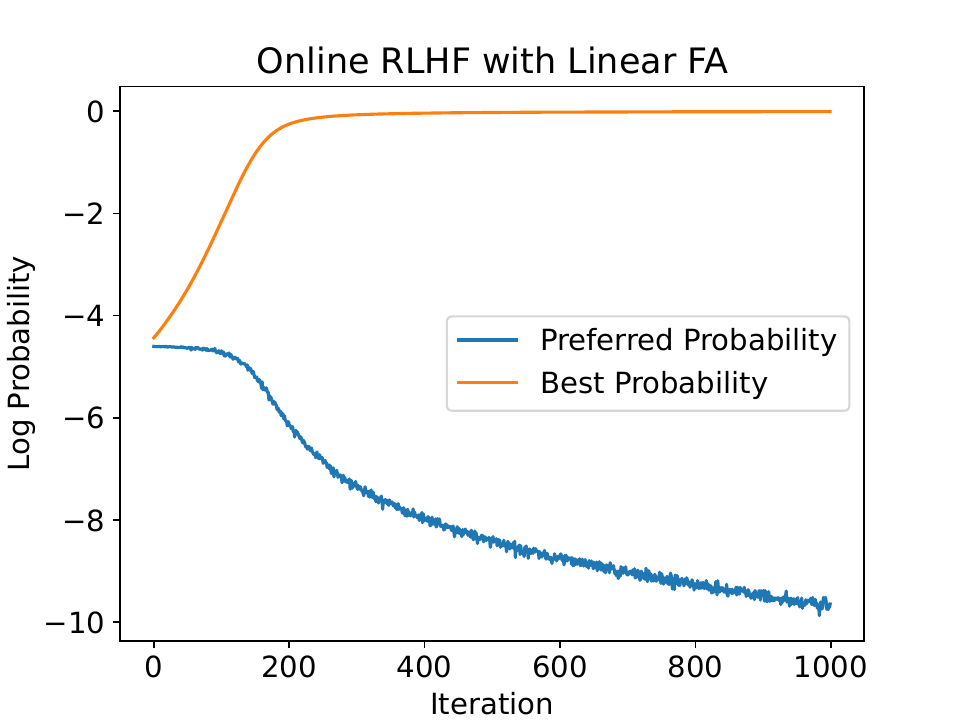}
    \includegraphics[width=0.3\linewidth]{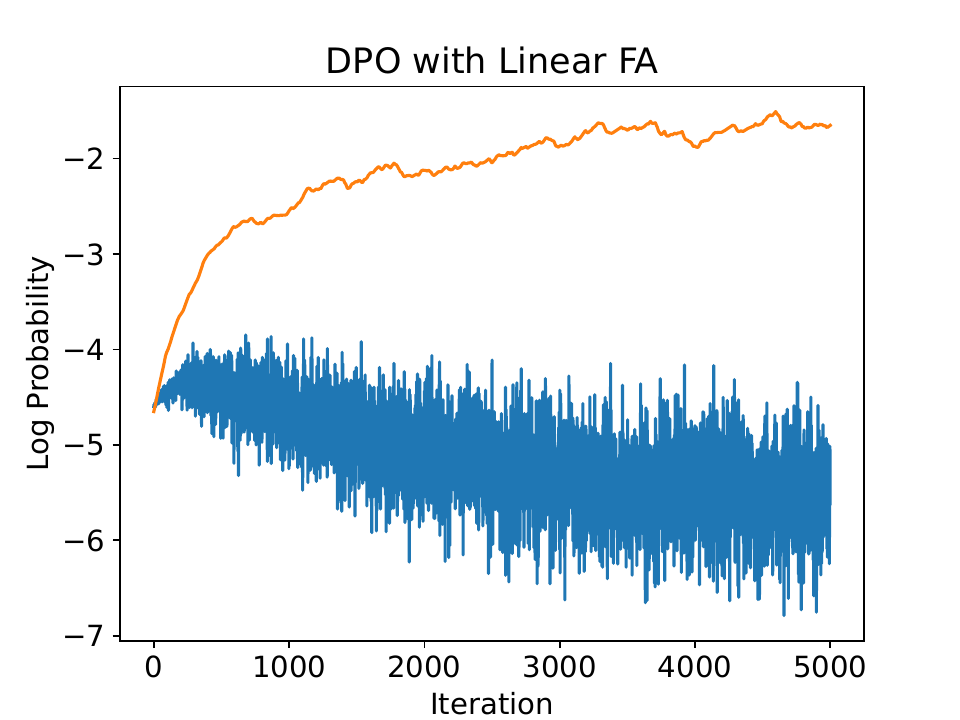}
    \includegraphics[width=0.3\linewidth]{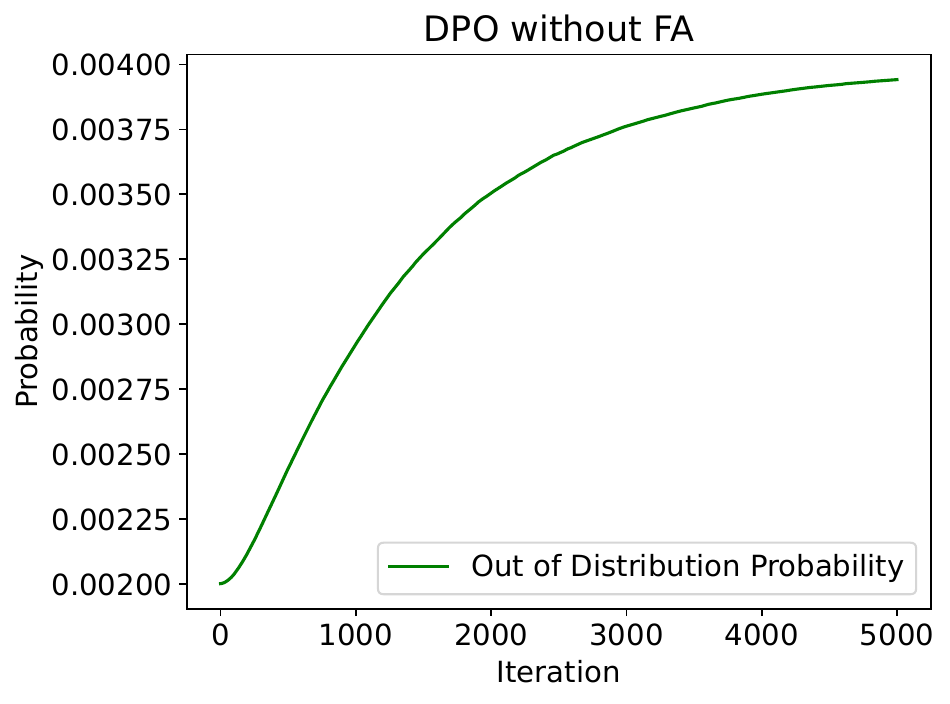}
    \caption{Left and middle: Extrapolation behavior of Online RL method and DPO under linear  function 
    approximation (FA). We plot the mean log probability of the preferred
    responses and the log probability of the best response, which is 
    unseen in the training data. We see that both algorithms correctly 
    assigns increasing probability to the best response.
    Right: Extrapolation behavior of DPO without function approximation. We plot the average probability of out-of-distribution responses along the training and DPO assigns increasing probability to out-of-distribution responses.}
    \label{fig:extrapolation}
\end{figure}

\subsubsection{Extrapolation without Function Approximation}
To further demonstrate the importance of function approximation and generalization, we conduct the same experiments but without the linear function approximation (i.e., treat each action as independent just like one would do in the classic multi-armed bandit setting). In this case, trying one action does not give us information about the other action. 
We see that in this case, DPO can erroneously
assign a higher probability to unseen suboptimal responses instead of the unseen optimal response, indicating DPO can fail to extrapolate and generalize correctly. 
Our investigation identifies that function approximation and generalization play an important role in the success of RLHF and DPO algorithms.

Now we describe the setting where function approximation fails, and this reduces 
to a Multi-arm bandit setting. We set $|\Ycal| = 500$, and the offline data only 
covers the first half of the responses. The $r^\ast(y)$ is set by sampling 
from Unif(-10,10), and we generate 10000 offline samples by uniformly sampling 
pairs of responses from the first half of the response space, and then labeling 
them with BT model under $r^\ast$. We train DPO with 5000 iterations, and plot 
the mean probability of the responses \emph{outside} of the data support in \pref{fig:extrapolation} (Right): we observe that the mean probability of the out-of-distribution responses is increasing, however, this could be an undesirable behavior because the reward of the out-of-distribution responses could be 
arbitrarily bad. 
}

\neurips{\vspace{-2mm}}
\section{Discussion}
\neurips{\vspace{-2mm}}

\neurips{
There are a few limitations of our work:
1) our theoretical analysis only considers the statistical perspective of 
each algorithm, but we believe our result is complementary to the 
other work that considers the optimization perspectives \citep{tajwar2024preference}. 2) we only conduct experiments on 
limited models and benchmarks. 
3) The experiment result shows that HyPO's performance is still below the one of online RLHF: this might suggest that our theory does not fully explain the benefit 
of all the components of online RLHF. For example, one hypothesis is that 
the learn reward function may have better generalization ability. 
4) It is not clear that the KL-ball coverage is necessary for online RL-based methods. 
However, as we discussed, since a bounded reverse KL might still induce exponential 
error amplification, we conjecture that at least the single policy coverage \citep{zhan2022offline} is not sufficient for online RLHF-based methods that use reverse KL. We believe these limitations lead to several interesting further directions. Finally, our method may not 
explicitly address the potential hallucinations or toxic behavior of LLMs, which 
is a common shortcoming of general-purpose fine-tuning algorithms.
}

\arxiv{
There are a few limitations of our work:
\begin{itemize}
    \item Our theoretical analysis only considers the statistical perspective of 
each algorithm, but we believe our result is complementary to the 
other work that considers the optimization perspectives \citep{tajwar2024preference}.
\item The experiment result shows that HyPO's performance is still below the one of online RLHF: this might suggest that our theory does not fully explain the benefit 
of all the components of online RLHF. For example, one hypothesis is that 
the learn reward function may have better generalization ability. 
\item It is not clear that the KL-ball coverage is necessary for online RL-based methods. 
However, as we discussed, since a bounded reverse KL might still induce exponential 
error amplification, we conjecture that at least the single policy coverage \citep{zhan2022offline} is not sufficient for online RLHF-based methods that use reverse KL.
\item We only use the online sample for regularizing reverse KL. One might 
also leverage the online sample to query new preference information 
or perform length control. 
\end{itemize}

We believe these limitations lead to several interesting further directions. Finally, our method may not 
explicitly address the potential hallucinations or toxic behavior of LLMs, which 
is a common shortcoming of general-purpose fine-tuning algorithms.
}

\arxiv{

\section*{Acknowledgments}
YS thanks Audrey Huang for the valuable discussions. AS and YS acknowledge and thank the support of ONR grant N000142212363 and NSF AI Institute for Societal Decision Making AI-SDM grant IIS2229881. WS acknowledges the support of NSF grant IIS-2154711 and NSF CAREER grant 2339395.
}

\newpage

\neurips{
\bibliographystyle{abbrv}  
}
\bibliography{refs} 

\begin{thebibliography}{66}
\providecommand{\natexlab}[1]{#1}
\providecommand{\url}[1]{\texttt{#1}}
\expandafter\ifx\csname urlstyle\endcsname\relax
  \providecommand{\doi}[1]{doi: #1}\else
  \providecommand{\doi}{doi: \begingroup \urlstyle{rm}\Url}\fi

\bibitem[Ahmadian et~al.(2024)Ahmadian, Cremer, Gall{\'e}, Fadaee, Kreutzer, {\"U}st{\"u}n, and Hooker]{ahmadian2024back}
Arash Ahmadian, Chris Cremer, Matthias Gall{\'e}, Marzieh Fadaee, Julia Kreutzer, Ahmet {\"U}st{\"u}n, and Sara Hooker.
\newblock Back to basics: Revisiting reinforce style optimization for learning from human feedback in llms.
\newblock \emph{arXiv preprint arXiv:2402.14740}, 2024.

\bibitem[Amortila et~al.(2024)Amortila, Foster, Jiang, Sekhari, and Xie]{amortila2024harnessing}
Philip Amortila, Dylan~J Foster, Nan Jiang, Ayush Sekhari, and Tengyang Xie.
\newblock Harnessing density ratios for online reinforcement learning.
\newblock In \emph{The Twelfth International Conference on Learning Representations}, 2024.
\newblock URL \url{https://openreview.net/forum?id=THJEa8adBn}.

\bibitem[Azar et~al.(2024)Azar, Guo, Piot, Munos, Rowland, Valko, and Calandriello]{azar2024general}
Mohammad~Gheshlaghi Azar, Zhaohan~Daniel Guo, Bilal Piot, Remi Munos, Mark Rowland, Michal Valko, and Daniele Calandriello.
\newblock A general theoretical paradigm to understand learning from human preferences.
\newblock In \emph{International Conference on Artificial Intelligence and Statistics}, pages 4447--4455. PMLR, 2024.

\bibitem[Bagnell et~al.(2003)Bagnell, Kakade, Schneider, and Ng]{bagnell2003policy}
James Bagnell, Sham~M Kakade, Jeff Schneider, and Andrew Ng.
\newblock Policy search by dynamic programming.
\newblock \emph{Advances in neural information processing systems}, 16, 2003.

\bibitem[Biderman et~al.(2023)Biderman, Schoelkopf, Anthony, Bradley, O’Brien, Hallahan, Khan, Purohit, Prashanth, Raff, et~al.]{biderman2023pythia}
Stella Biderman, Hailey Schoelkopf, Quentin~Gregory Anthony, Herbie Bradley, Kyle O’Brien, Eric Hallahan, Mohammad~Aflah Khan, Shivanshu Purohit, USVSN~Sai Prashanth, Edward Raff, et~al.
\newblock Pythia: A suite for analyzing large language models across training and scaling.
\newblock In \emph{International Conference on Machine Learning}, pages 2397--2430. PMLR, 2023.

\bibitem[Bradley and Terry(1952)]{bradley1952rank}
Ralph~Allan Bradley and Milton~E Terry.
\newblock Rank analysis of incomplete block designs: I. the method of paired comparisons.
\newblock \emph{Biometrika}, 39\penalty0 (3/4):\penalty0 324--345, 1952.

\bibitem[Casper et~al.(2023)Casper, Davies, Shi, Gilbert, Scheurer, Rando, Freedman, Korbak, Lindner, Freire, et~al.]{casper2023open}
Stephen Casper, Xander Davies, Claudia Shi, Thomas~Krendl Gilbert, J{\'e}r{\'e}my Scheurer, Javier Rando, Rachel Freedman, Tomasz Korbak, David Lindner, Pedro Freire, et~al.
\newblock Open problems and fundamental limitations of reinforcement learning from human feedback.
\newblock \emph{arXiv preprint arXiv:2307.15217}, 2023.

\bibitem[Chang et~al.(2023)Chang, Brantley, Ramamurthy, Misra, and Sun]{chang2023learning}
Jonathan~D Chang, Kiante Brantley, Rajkumar Ramamurthy, Dipendra Misra, and Wen Sun.
\newblock Learning to generate better than your llm.
\newblock \emph{arXiv preprint arXiv:2306.11816}, 2023.

\bibitem[Chang et~al.(2024)Chang, Shan, Oertell, Brantley, Misra, Lee, and Sun]{chang2024dataset}
Jonathan~D Chang, Wenhao Shan, Owen Oertell, Kiant{\'e} Brantley, Dipendra Misra, Jason~D Lee, and Wen Sun.
\newblock Dataset reset policy optimization for rlhf.
\newblock \emph{arXiv preprint arXiv:2404.08495}, 2024.

\bibitem[Chen and Jiang(2022)]{chen2022offline}
Jinglin Chen and Nan Jiang.
\newblock Offline reinforcement learning under value and density-ratio realizability: the power of gaps.
\newblock In \emph{Uncertainty in Artificial Intelligence}, pages 378--388. PMLR, 2022.

\bibitem[Christiano et~al.(2017)Christiano, Leike, Brown, Martic, Legg, and Amodei]{christiano2017deep}
Paul~F Christiano, Jan Leike, Tom Brown, Miljan Martic, Shane Legg, and Dario Amodei.
\newblock Deep reinforcement learning from human preferences.
\newblock \emph{Advances in neural information processing systems}, 30, 2017.

\bibitem[Clark et~al.(2018)Clark, Cowhey, Etzioni, Khot, Sabharwal, Schoenick, and Tafjord]{clark2018think}
Peter Clark, Isaac Cowhey, Oren Etzioni, Tushar Khot, Ashish Sabharwal, Carissa Schoenick, and Oyvind Tafjord.
\newblock Think you have solved question answering? try arc, the ai2 reasoning challenge.
\newblock \emph{arXiv preprint arXiv:1803.05457}, 2018.

\bibitem[Cobbe et~al.(2021)Cobbe, Kosaraju, Bavarian, Chen, Jun, Kaiser, Plappert, Tworek, Hilton, Nakano, et~al.]{cobbe2021training}
Karl Cobbe, Vineet Kosaraju, Mohammad Bavarian, Mark Chen, Heewoo Jun, Lukasz Kaiser, Matthias Plappert, Jerry Tworek, Jacob Hilton, Reiichiro Nakano, et~al.
\newblock Training verifiers to solve math word problems.
\newblock \emph{arXiv preprint arXiv:2110.14168}, 2021.

\bibitem[Cui et~al.(2023)Cui, Yuan, Ding, Yao, Zhu, Ni, Xie, Liu, and Sun]{cui2023ultrafeedback}
Ganqu Cui, Lifan Yuan, Ning Ding, Guanming Yao, Wei Zhu, Yuan Ni, Guotong Xie, Zhiyuan Liu, and Maosong Sun.
\newblock Ultrafeedback: Boosting language models with high-quality feedback.
\newblock \emph{arXiv preprint arXiv:2310.01377}, 2023.

\bibitem[Dubois et~al.(2024)Dubois, Galambosi, Liang, and Hashimoto]{dubois2024length}
Yann Dubois, Bal{\'a}zs Galambosi, Percy Liang, and Tatsunori~B Hashimoto.
\newblock Length-controlled alpacaeval: A simple way to debias automatic evaluators.
\newblock \emph{arXiv preprint arXiv:2404.04475}, 2024.

\bibitem[Eisenstein et~al.(2023)Eisenstein, Nagpal, Agarwal, Beirami, D'Amour, Dvijotham, Fisch, Heller, Pfohl, Ramachandran, et~al.]{eisenstein2023helping}
Jacob Eisenstein, Chirag Nagpal, Alekh Agarwal, Ahmad Beirami, Alex D'Amour, DJ~Dvijotham, Adam Fisch, Katherine Heller, Stephen Pfohl, Deepak Ramachandran, et~al.
\newblock Helping or herding? reward model ensembles mitigate but do not eliminate reward hacking.
\newblock \emph{arXiv preprint arXiv:2312.09244}, 2023.

\bibitem[Fisch et~al.(2024)Fisch, Eisenstein, Zayats, Agarwal, Beirami, Nagpal, Shaw, and Berant]{fisch2024robust}
Adam Fisch, Jacob Eisenstein, Vicky Zayats, Alekh Agarwal, Ahmad Beirami, Chirag Nagpal, Pete Shaw, and Jonathan Berant.
\newblock Robust preference optimization through reward model distillation.
\newblock \emph{arXiv preprint arXiv:2405.19316}, 2024.

\bibitem[Gao et~al.(2023)Gao, Schulman, and Hilton]{gao2023scaling}
Leo Gao, John Schulman, and Jacob Hilton.
\newblock Scaling laws for reward model overoptimization.
\newblock In \emph{International Conference on Machine Learning}, pages 10835--10866. PMLR, 2023.

\bibitem[Gao et~al.(2024)Gao, Chang, Zhan, Oertell, Swamy, Brantley, Joachims, Bagnell, Lee, and Sun]{gao2024rebel}
Zhaolin Gao, Jonathan~D Chang, Wenhao Zhan, Owen Oertell, Gokul Swamy, Kiant{\'e} Brantley, Thorsten Joachims, J~Andrew Bagnell, Jason~D Lee, and Wen Sun.
\newblock Rebel: Reinforcement learning via regressing relative rewards.
\newblock \emph{arXiv preprint arXiv:2404.16767}, 2024.

\bibitem[Guo et~al.(2024)Guo, Zhang, Liu, Liu, Khalman, Llinares, Rame, Mesnard, Zhao, Piot, et~al.]{guo2024direct}
Shangmin Guo, Biao Zhang, Tianlin Liu, Tianqi Liu, Misha Khalman, Felipe Llinares, Alexandre Rame, Thomas Mesnard, Yao Zhao, Bilal Piot, et~al.
\newblock Direct language model alignment from online ai feedback.
\newblock \emph{arXiv preprint arXiv:2402.04792}, 2024.

\bibitem[Hendrycks et~al.(2020)Hendrycks, Burns, Basart, Zou, Mazeika, Song, and Steinhardt]{hendrycks2020measuring}
Dan Hendrycks, Collin Burns, Steven Basart, Andy Zou, Mantas Mazeika, Dawn Song, and Jacob Steinhardt.
\newblock Measuring massive multitask language understanding.
\newblock \emph{arXiv preprint arXiv:2009.03300}, 2020.

\bibitem[Hu et~al.(2021)Hu, Shen, Wallis, Allen-Zhu, Li, Wang, Wang, and Chen]{hu2021lora}
Edward~J Hu, Yelong Shen, Phillip Wallis, Zeyuan Allen-Zhu, Yuanzhi Li, Shean Wang, Lu~Wang, and Weizhu Chen.
\newblock Lora: Low-rank adaptation of large language models.
\newblock \emph{arXiv preprint arXiv:2106.09685}, 2021.

\bibitem[Kakade and Langford(2002)]{kakade2002approximately}
Sham Kakade and John Langford.
\newblock Approximately optimal approximate reinforcement learning.
\newblock In \emph{Proceedings of the Nineteenth International Conference on Machine Learning}, pages 267--274, 2002.

\bibitem[Kirk et~al.(2023)Kirk, Mediratta, Nalmpantis, Luketina, Hambro, Grefenstette, and Raileanu]{kirk2023understanding}
Robert Kirk, Ishita Mediratta, Christoforos Nalmpantis, Jelena Luketina, Eric Hambro, Edward Grefenstette, and Roberta Raileanu.
\newblock Understanding the effects of rlhf on llm generalisation and diversity.
\newblock \emph{arXiv preprint arXiv:2310.06452}, 2023.

\bibitem[Kool et~al.(2019)Kool, van Hoof, and Welling]{kool2019buy}
Wouter Kool, Herke van Hoof, and Max Welling.
\newblock Buy 4 {REINFORCE} samples, get a baseline for free!
\newblock 2019.
\newblock URL \url{https://openreview.net/forum?id=r1lgTGL5DE}.

\bibitem[Langford and Zhang(2008)]{langford2008epoch}
John Langford and Tong Zhang.
\newblock The epoch-greedy algorithm for multi-armed bandits with side information.
\newblock In \emph{Advances in neural information processing systems}, pages 817--824, 2008.

\bibitem[Lin et~al.(2021)Lin, Hilton, and Evans]{lin2021truthfulqa}
Stephanie Lin, Jacob Hilton, and Owain Evans.
\newblock Truthfulqa: Measuring how models mimic human falsehoods.
\newblock \emph{arXiv preprint arXiv:2109.07958}, 2021.

\bibitem[Liu et~al.(2024)Liu, Lu, Zhang, Liu, Guo, Yang, Blanchet, and Wang]{liu2024provably}
Zhihan Liu, Miao Lu, Shenao Zhang, Boyi Liu, Hongyi Guo, Yingxiang Yang, Jose Blanchet, and Zhaoran Wang.
\newblock Provably mitigating overoptimization in rlhf: Your sft loss is implicitly an adversarial regularizer.
\newblock \emph{arXiv preprint arXiv:2405.16436}, 2024.

\bibitem[Meta(2024)]{meta2024introducing}
Meta.
\newblock Introducing meta llama 3: The most capable openly available llm to date, 2024.
\newblock \emph{URL https://ai. meta. com/blog/meta-llama-3}, 2024.

\bibitem[Munos and Szepesv{\'a}ri(2008)]{munos2008finite}
R{\'e}mi Munos and Csaba Szepesv{\'a}ri.
\newblock Finite-time bounds for fitted value iteration.
\newblock \emph{Journal of Machine Learning Research}, 9\penalty0 (5), 2008.

\bibitem[Munos et~al.(2023)Munos, Valko, Calandriello, Azar, Rowland, Guo, Tang, Geist, Mesnard, Michi, et~al.]{munos2023nash}
R{\'e}mi Munos, Michal Valko, Daniele Calandriello, Mohammad~Gheshlaghi Azar, Mark Rowland, Zhaohan~Daniel Guo, Yunhao Tang, Matthieu Geist, Thomas Mesnard, Andrea Michi, et~al.
\newblock Nash learning from human feedback.
\newblock \emph{arXiv preprint arXiv:2312.00886}, 2023.

\bibitem[Ouyang et~al.(2022)Ouyang, Wu, Jiang, Almeida, Wainwright, Mishkin, Zhang, Agarwal, Slama, Ray, et~al.]{ouyang2022training}
Long Ouyang, Jeffrey Wu, Xu~Jiang, Diogo Almeida, Carroll Wainwright, Pamela Mishkin, Chong Zhang, Sandhini Agarwal, Katarina Slama, Alex Ray, et~al.
\newblock Training language models to follow instructions with human feedback.
\newblock \emph{Advances in neural information processing systems}, 35:\penalty0 27730--27744, 2022.

\bibitem[Pal et~al.(2024)Pal, Karkhanis, Dooley, Roberts, Naidu, and White]{pal2024smaug}
Arka Pal, Deep Karkhanis, Samuel Dooley, Manley Roberts, Siddartha Naidu, and Colin White.
\newblock Smaug: Fixing failure modes of preference optimisation with dpo-positive.
\newblock \emph{arXiv preprint arXiv:2402.13228}, 2024.

\bibitem[Rafailov et~al.(2023)Rafailov, Sharma, Mitchell, Manning, Ermon, and Finn]{rafailov2024direct}
Rafael Rafailov, Archit Sharma, Eric Mitchell, Christopher~D Manning, Stefano Ermon, and Chelsea Finn.
\newblock Direct preference optimization: Your language model is secretly a reward model.
\newblock \emph{Advances in Neural Information Processing Systems}, 36, 2023.

\bibitem[Rafailov et~al.(2024)Rafailov, Hejna, Park, and Finn]{rafailov2024r}
Rafael Rafailov, Joey Hejna, Ryan Park, and Chelsea Finn.
\newblock From $ r $ to $ q^{}* $: Your language model is secretly a q-function.
\newblock \emph{arXiv preprint arXiv:2404.12358}, 2024.

\bibitem[Ross and Bagnell(2012)]{ross2012agnostic}
Stephane Ross and J~Andrew Bagnell.
\newblock Agnostic system identification for model-based reinforcement learning.
\newblock \emph{arXiv preprint arXiv:1203.1007}, 2012.

\bibitem[Rosset et~al.(2024)Rosset, Cheng, Mitra, Santacroce, Awadallah, and Xie]{rosset2024direct}
Corby Rosset, Ching-An Cheng, Arindam Mitra, Michael Santacroce, Ahmed Awadallah, and Tengyang Xie.
\newblock Direct nash optimization: Teaching language models to self-improve with general preferences.
\newblock \emph{arXiv preprint arXiv:2404.03715}, 2024.

\bibitem[Schulman et~al.(2017)Schulman, Wolski, Dhariwal, Radford, and Klimov]{schulman2017proximal}
John Schulman, Filip Wolski, Prafulla Dhariwal, Alec Radford, and Oleg Klimov.
\newblock Proximal policy optimization algorithms.
\newblock \emph{arXiv preprint arXiv:1707.06347}, 2017.

\bibitem[Sharma et~al.(2024)Sharma, Keh, Mitchell, Finn, Arora, and Kollar]{sharma2024critical}
Archit Sharma, Sedrick Keh, Eric Mitchell, Chelsea Finn, Kushal Arora, and Thomas Kollar.
\newblock A critical evaluation of ai feedback for aligning large language models.
\newblock \emph{arXiv preprint arXiv:2402.12366}, 2024.

\bibitem[Singhal et~al.(2023)Singhal, Goyal, Xu, and Durrett]{singhal2023long}
Prasann Singhal, Tanya Goyal, Jiacheng Xu, and Greg Durrett.
\newblock A long way to go: Investigating length correlations in rlhf.
\newblock \emph{arXiv preprint arXiv:2310.03716}, 2023.

\bibitem[Song et~al.(2022)Song, Zhou, Sekhari, Bagnell, Krishnamurthy, and Sun]{song2022hybrid}
Yuda Song, Yifei Zhou, Ayush Sekhari, J~Andrew Bagnell, Akshay Krishnamurthy, and Wen Sun.
\newblock Hybrid rl: Using both offline and online data can make rl efficient.
\newblock \emph{arXiv preprint arXiv:2210.06718}, 2022.

\bibitem[Stiennon et~al.(2020)Stiennon, Ouyang, Wu, Ziegler, Lowe, Voss, Radford, Amodei, and Christiano]{stiennon2020learning}
Nisan Stiennon, Long Ouyang, Jeffrey Wu, Daniel Ziegler, Ryan Lowe, Chelsea Voss, Alec Radford, Dario Amodei, and Paul~F Christiano.
\newblock Learning to summarize with human feedback.
\newblock \emph{Advances in Neural Information Processing Systems}, 33:\penalty0 3008--3021, 2020.

\bibitem[Swamy et~al.(2024)Swamy, Dann, Kidambi, Wu, and Agarwal]{swamy2024minimaximalist}
Gokul Swamy, Christoph Dann, Rahul Kidambi, Zhiwei~Steven Wu, and Alekh Agarwal.
\newblock A minimaximalist approach to reinforcement learning from human feedback.
\newblock \emph{arXiv preprint arXiv:2401.04056}, 2024.

\bibitem[Tajwar et~al.(2024)Tajwar, Singh, Sharma, Rafailov, Schneider, Xie, Ermon, Finn, and Kumar]{tajwar2024preference}
Fahim Tajwar, Anikait Singh, Archit Sharma, Rafael Rafailov, Jeff Schneider, Tengyang Xie, Stefano Ermon, Chelsea Finn, and Aviral Kumar.
\newblock Preference fine-tuning of llms should leverage suboptimal, on-policy data.
\newblock \emph{arXiv preprint arXiv:2404.14367}, 2024.

\bibitem[Tang et~al.(2024)Tang, Guo, Zheng, Calandriello, Cao, Tarassov, Munos, Pires, Valko, Cheng, et~al.]{tang2024understanding}
Yunhao Tang, Daniel~Zhaohan Guo, Zeyu Zheng, Daniele Calandriello, Yuan Cao, Eugene Tarassov, R{\'e}mi Munos, Bernardo~{\'A}vila Pires, Michal Valko, Yong Cheng, et~al.
\newblock Understanding the performance gap between online and offline alignment algorithms.
\newblock \emph{arXiv preprint arXiv:2405.08448}, 2024.

\bibitem[Team et~al.(2023)Team, Anil, Borgeaud, Wu, Alayrac, Yu, Soricut, Schalkwyk, Dai, Hauth, et~al.]{team2023gemini}
Gemini Team, Rohan Anil, Sebastian Borgeaud, Yonghui Wu, Jean-Baptiste Alayrac, Jiahui Yu, Radu Soricut, Johan Schalkwyk, Andrew~M Dai, Anja Hauth, et~al.
\newblock Gemini: a family of highly capable multimodal models.
\newblock \emph{arXiv preprint arXiv:2312.11805}, 2023.

\bibitem[Touvron et~al.(2023)Touvron, Martin, Stone, Albert, Almahairi, Babaei, Bashlykov, Batra, Bhargava, Bhosale, et~al.]{touvron2023llama}
Hugo Touvron, Louis Martin, Kevin Stone, Peter Albert, Amjad Almahairi, Yasmine Babaei, Nikolay Bashlykov, Soumya Batra, Prajjwal Bhargava, Shruti Bhosale, et~al.
\newblock Llama 2: Open foundation and fine-tuned chat models.
\newblock \emph{arXiv preprint arXiv:2307.09288}, 2023.

\bibitem[Uehara and Sun(2021)]{uehara2021pessimistic}
Masatoshi Uehara and Wen Sun.
\newblock Pessimistic model-based offline reinforcement learning under partial coverage.
\newblock \emph{arXiv preprint arXiv:2107.06226}, 2021.

\bibitem[von Werra et~al.(2020)von Werra, Belkada, Tunstall, Beeching, Thrush, Lambert, and Huang]{vonwerra2022trl}
Leandro von Werra, Younes Belkada, Lewis Tunstall, Edward Beeching, Tristan Thrush, Nathan Lambert, and Shengyi Huang.
\newblock Trl: Transformer reinforcement learning.
\newblock \url{https://github.com/huggingface/trl}, 2020.

\bibitem[Williams(1992)]{williams1992simple}
Ronald~J Williams.
\newblock Simple statistical gradient-following algorithms for connectionist reinforcement learning.
\newblock \emph{Machine learning}, 8:\penalty0 229--256, 1992.

\bibitem[Xie et~al.(2021{\natexlab{a}})Xie, Cheng, Jiang, Mineiro, and Agarwal]{xie2021bellman}
Tengyang Xie, Ching-An Cheng, Nan Jiang, Paul Mineiro, and Alekh Agarwal.
\newblock Bellman-consistent pessimism for offline reinforcement learning.
\newblock \emph{Advances in neural information processing systems}, 34:\penalty0 6683--6694, 2021{\natexlab{a}}.

\bibitem[Xie et~al.(2021{\natexlab{b}})Xie, Jiang, Wang, Xiong, and Bai]{xie2021policy}
Tengyang Xie, Nan Jiang, Huan Wang, Caiming Xiong, and Yu~Bai.
\newblock Policy finetuning: Bridging sample-efficient offline and online reinforcement learning.
\newblock \emph{Advances in neural information processing systems}, 34:\penalty0 27395--27407, 2021{\natexlab{b}}.

\bibitem[Xie et~al.(2023)Xie, Foster, Bai, Jiang, and Kakade]{xie2023the}
Tengyang Xie, Dylan~J Foster, Yu~Bai, Nan Jiang, and Sham~M. Kakade.
\newblock The role of coverage in online reinforcement learning.
\newblock In \emph{The Eleventh International Conference on Learning Representations}, 2023.
\newblock URL \url{https://openreview.net/forum?id=LQIjzPdDt3q}.

\bibitem[Xiong et~al.(2022)Xiong, Zhong, Shi, Shen, Wang, and Zhang]{xiong2022nearly}
Wei Xiong, Han Zhong, Chengshuai Shi, Cong Shen, Liwei Wang, and Tong Zhang.
\newblock Nearly minimax optimal offline reinforcement learning with linear function approximation: Single-agent mdp and markov game.
\newblock \emph{arXiv preprint arXiv:2205.15512}, 2022.

\bibitem[Xiong et~al.(2023)Xiong, Dong, Ye, Wang, Zhong, Ji, Jiang, and Zhang]{xiong2023iterative}
Wei Xiong, Hanze Dong, Chenlu Ye, Ziqi Wang, Han Zhong, Heng Ji, Nan Jiang, and Tong Zhang.
\newblock Iterative preference learning from human feedback: Bridging theory and practice for rlhf under kl-constraint.
\newblock In \emph{ICLR 2024 Workshop on Mathematical and Empirical Understanding of Foundation Models}, 2023.

\bibitem[Xiong et~al.(2024)Xiong, Dong, Ye, Wang, Zhong, Ji, Jiang, and Zhang]{xiong2024iterative}
Wei Xiong, Hanze Dong, Chenlu Ye, Ziqi Wang, Han Zhong, Heng Ji, Nan Jiang, and Tong Zhang.
\newblock Iterative preference learning from human feedback: Bridging theory and practice for rlhf under kl-constraint.
\newblock In \emph{Forty-first International Conference on Machine Learning}, 2024.

\bibitem[Xu et~al.(2024{\natexlab{a}})Xu, Sharaf, Chen, Tan, Shen, Van~Durme, Murray, and Kim]{xu2024contrastive}
Haoran Xu, Amr Sharaf, Yunmo Chen, Weiting Tan, Lingfeng Shen, Benjamin Van~Durme, Kenton Murray, and Young~Jin Kim.
\newblock Contrastive preference optimization: Pushing the boundaries of llm performance in machine translation.
\newblock \emph{arXiv preprint arXiv:2401.08417}, 2024{\natexlab{a}}.

\bibitem[Xu et~al.(2024{\natexlab{b}})Xu, Fu, Gao, Ye, Liu, Mei, Wang, Yu, and Wu]{xu2024dpo}
Shusheng Xu, Wei Fu, Jiaxuan Gao, Wenjie Ye, Weilin Liu, Zhiyu Mei, Guangju Wang, Chao Yu, and Yi~Wu.
\newblock Is dpo superior to ppo for llm alignment? a comprehensive study.
\newblock \emph{arXiv preprint arXiv:2404.10719}, 2024{\natexlab{b}}.

\bibitem[Yuan et~al.(2024)Yuan, Cui, Wang, Ding, Wang, Deng, Shan, Chen, Xie, Lin, et~al.]{yuan2024advancing}
Lifan Yuan, Ganqu Cui, Hanbin Wang, Ning Ding, Xingyao Wang, Jia Deng, Boji Shan, Huimin Chen, Ruobing Xie, Yankai Lin, et~al.
\newblock Advancing llm reasoning generalists with preference trees.
\newblock \emph{arXiv preprint arXiv:2404.02078}, 2024.

\bibitem[Zellers et~al.(2019)Zellers, Holtzman, Bisk, Farhadi, and Choi]{zellers2019hellaswag}
Rowan Zellers, Ari Holtzman, Yonatan Bisk, Ali Farhadi, and Yejin Choi.
\newblock Hellaswag: Can a machine really finish your sentence?
\newblock \emph{arXiv preprint arXiv:1905.07830}, 2019.

\bibitem[Zhan et~al.(2022)Zhan, Huang, Huang, Jiang, and Lee]{zhan2022offline}
Wenhao Zhan, Baihe Huang, Audrey Huang, Nan Jiang, and Jason Lee.
\newblock Offline reinforcement learning with realizability and single-policy concentrability.
\newblock In \emph{Conference on Learning Theory}, pages 2730--2775. PMLR, 2022.

\bibitem[Zhan et~al.(2023)Zhan, Uehara, Kallus, Lee, and Sun]{zhan2023provable}
Wenhao Zhan, Masatoshi Uehara, Nathan Kallus, Jason~D Lee, and Wen Sun.
\newblock Provable offline preference-based reinforcement learning.
\newblock In \emph{The Twelfth International Conference on Learning Representations}, 2023.

\bibitem[Zhao et~al.(2023)Zhao, Joshi, Liu, Khalman, Saleh, and Liu]{zhao2023slic}
Yao Zhao, Rishabh Joshi, Tianqi Liu, Misha Khalman, Mohammad Saleh, and Peter~J Liu.
\newblock Slic-hf: Sequence likelihood calibration with human feedback.
\newblock \emph{arXiv preprint arXiv:2305.10425}, 2023.

\bibitem[Zheng et~al.(2024)Zheng, Chiang, Sheng, Zhuang, Wu, Zhuang, Lin, Li, Li, Xing, et~al.]{zheng2024judging}
Lianmin Zheng, Wei-Lin Chiang, Ying Sheng, Siyuan Zhuang, Zhanghao Wu, Yonghao Zhuang, Zi~Lin, Zhuohan Li, Dacheng Li, Eric Xing, et~al.
\newblock Judging llm-as-a-judge with mt-bench and chatbot arena.
\newblock \emph{Advances in Neural Information Processing Systems}, 36, 2024.

\bibitem[Zhu et~al.(2023)Zhu, Sharma, Frujeri, Dong, Zhu, Jordan, and Jiao]{zhu2023fine}
Banghua Zhu, Hiteshi Sharma, Felipe~Vieira Frujeri, Shi Dong, Chenguang Zhu, Michael~I Jordan, and Jiantao Jiao.
\newblock Fine-tuning language models with advantage-induced policy alignment.
\newblock \emph{arXiv preprint arXiv:2306.02231}, 2023.

\bibitem[Ziebart et~al.(2008)Ziebart, Maas, Bagnell, Dey, et~al.]{ziebart2008maximum}
Brian~D Ziebart, Andrew~L Maas, J~Andrew Bagnell, Anind~K Dey, et~al.
\newblock Maximum entropy inverse reinforcement learning.
\newblock In \emph{Aaai}, volume~8, pages 1433--1438. Chicago, IL, USA, 2008.

\end{thebibliography}
\newpage
\appendix
\section{Auxiliary Lemmas}

\begin{lemma}[Objective decomposition]\label{lem:objective_decomposition}
    Let $J(\pi)$ be the objective function defined in \eqref{eq:objective}, and 
    for reward function $\hat r$, we let 
    \begin{align}\label{eq:optimal_policy}
        \hat \pi \in \argmax_{\pi} \EE_{x \sim \rho} \brk*{\EE_{y \sim \pi(\cdot \mid x)} [\hat{r}(x,y)] - \beta \KL\prn*{\pi(\cdot \mid x) || \piref(\cdot \mid x)}},
    \end{align}
    then we have 
    \begin{align*}
        J(\pi^\ast) - J(\hat \pi) \leq 
        \EE_{x \sim \rho} \brk*{ \EE_{y^1 \sim \pi^\ast(\cdot \mid x), y^2 \sim \hat \pi(\cdot \mid x)} 
        \brk*{r^\ast(x,y^1) - \hat r(x,y^1) - r^\ast(x,y^2) + \hat r(x,y^2)}}.
    \end{align*}
\end{lemma}

\begin{proof}
We have 
\begin{align*}
    &J(\pi^\ast) - J(\hat \pi) \\ =& \EE_{x \sim \rho} \brk*{\EE_{y \sim \pi^\ast(\cdot \mid x)} [r^\ast(x,y)] - \beta \KL(\pi^\ast(\cdot \mid x) || \piref(\cdot \mid x))} - \EE_{x \sim \rho} \brk*{ \EE_{y \sim \hat \pi(\cdot \mid x)} [r^\ast(x,y)]  + \beta \KL(\hat \pi(\cdot \mid x) || \piref(\cdot \mid x)) }\\
    =& \EE_{x \sim \rho} \brk*{\EE_{y \sim \pi^\ast(\cdot \mid x)} [r^\ast(x,y)] - \beta \KL(\pi^\ast(\cdot \mid x) || \piref(\cdot \mid x))}- \prn*{\EE_{x \sim \rho} \brk*{ \EE_{y \sim \hat \pi(\cdot \mid x)} [ r^\ast(x,y)]  - \beta \KL(\hat \pi(\cdot \mid x) || \piref(\cdot \mid x))}} \\
    &+ \EE_{x \sim \rho} \brk*{\EE_{y \sim \hat \pi(\cdot \mid x)} [\hat r(x,y)] - \beta \KL(\hat \pi(\cdot \mid x) || \piref(\cdot \mid x))} - \prn*{ \EE_{x \sim \rho} \brk*{ \EE_{y \sim \hat \pi(\cdot \mid x)} [\hat r(x,y)]  - \beta \KL(\hat \pi(\cdot \mid x) || \piref(\cdot \mid x)) }} \\
    \leq& \EE_{x \sim \rho} \brk*{ \EE_{y \sim \pi^\ast(\cdot \mid x)} [r^\ast(x,y)] - \beta \KL(\pi^\ast(\cdot \mid x) || \piref(\cdot \mid x))} -\prn*{\EE_{x \sim \rho} \brk*{ \EE_{y \sim \pi^\ast(\cdot \mid x)} [\hat r(x,y)]  - \beta \KL(\pi^\ast(\cdot \mid x) || \piref(\cdot \mid x))}} \\
    &+ \EE_{x \sim \rho} \brk*{ \EE_{y \sim \hat \pi(\cdot \mid x)} [\hat r(x,y)] - \beta \KL(\hat \pi(\cdot \mid x) || \piref(\cdot \mid x))} - \prn*{ \EE_{x \sim \rho} \brk*{\EE_{y \sim \hat \pi(\cdot \mid x)} [ r^\ast(x,y)]  - \beta \KL(\hat \pi(\cdot \mid x) || \piref(\cdot \mid x)) }} \\
    =& \EE_{x \sim \rho} \brk*{ \EE_{y \sim \pi^\ast(\cdot \mid x)} [r^\ast(x,y) - \hat r(x,y)]} - \EE_{x \sim \rho} \brk*{\EE_{y \sim \hat \pi(\cdot \mid x)} [r^\ast(x,y) - \hat r(x,y)]},
\end{align*}
where the inequality is due to \pref{eq:optimal_policy}. To complete the proof, note that 
\begin{align*}
    &\EE_{x \sim \rho} \brk*{ \EE_{y \sim \pi^\ast(\cdot \mid x)} [r^\ast(x,y) - \hat r(x,y)] - \EE_{x \sim \rho} \EE_{y \sim \hat \pi(\cdot \mid x)} [r^\ast(x,y) - \hat r(x,y)] }\\
    =& \EE_{x \sim \rho} \brk*{ \EE_{y^1 \sim \pi^\ast(\cdot \mid x), y^2 \sim \hat \pi(\cdot \mid x)}[r^\ast(x,y^1) - \hat r(x,y^1)]} - \EE_{x \sim \rho} \brk*{ \EE_{y^1 \sim \pi^\ast(\cdot \mid x), y^2 \sim \hat \pi(\cdot \mid x)}[r^\ast(x,y^2) - \hat r(x,y^2)]} \\
    =& \EE_{x \sim \rho} \brk*{ \EE_{y^1 \sim \pi^\ast(\cdot \mid x), y^2 \sim \hat \pi(\cdot \mid x)}
    \brk*{r^\ast(x,y^1) - \hat r(x,y^1) - r^\ast(x,y^2) + \hat r(x,y^2)}}.
\end{align*}
\end{proof}

\begin{lemma}[Lemma C.2 from \citep{chang2024dataset}]
    \label{lem:reward_relative}
    Assume that $r^\ast$ is bounded, let $\Rcal$ be the reward function class, and
    Let
    \begin{align*}
        \hat{r} = \argmin_{r \in \Rcal}  \hat \EE_{x,y^+, y^- \sim \Dcal} \brk*{ \log 
        \prn*{\frac{\exp( r(x,y^+))}{\exp( r(x,y^+)) + \exp( r(x,y^-))}}},
    \end{align*}
    then we have with probability at least $1 - \delta$ that
    \begin{align*}
      \EE_{x,y^1,y^2 \sim \mu \circ \piref} \brk*{ \prn*{ r^\ast(x,y^1) - r^\ast(x,y^2) - \hat r(x,y^1) + \hat r(x,y^2)}^2 }\leq \frac{c \kappa^2 \log(|\Rcal| / \delta)}{N},
    \end{align*} 
    where $\kappa$ measures the non-linearity of the link function, and $c$ is a constant, $N := |\Dcal|$
    is the size of the offline dataset.
  \end{lemma}

\section{Additional Results}
\subsection{Results for IPO}\label{app:ipo}
In this section we give detailed technical details for IPO, and the negative results 
for IPO under partial coverage. Recall that the empirical objective of IPO is
is $\piipo \in \argmin_{\pi} \widehat \lipo(\pi)$, where
\begin{align*}
  \widehat \lipo(\pi) =  \widehat \EE_{x,y^+, y^- \sim \Dcal} \brk*{\prn*{\log \prn*{\frac{\pi(y^+ \mid x) \piref(y^- \mid x)}{\pi(y^- \mid x) \piref(y^+ \mid x)}} - \frac{\beta^{-1}}{2}}^2}.
\end{align*}
The empirical objective is derived from the following population loss 
\begin{align}\label{eq:lipo}
   \lipo(\pi) = \EE_{x,y^1,y^2 \sim \rho \circ \piref} \brk*{\prn*{h_\pi\prn*{y^1,y^2} - I\prn*{y^1,y^2}/\beta}^2},
\end{align}
where 
\begin{align*}
    h_\pi(y^1,y^2) = \log \prn*{\frac{\pi(y^1) \piref(y_2)}{\pi(y^2)\piref(y_1)}},
\end{align*}
and $I(y^1,y^2)$ is a Bernoulli random variable with parameter $p = p^\ast(y_1 \succ y_2)$,
where here $p^\ast$ can be any underlying human preference (that is not necessarily parametrized by the Bradley Terry model). To show the negative result, we can make the following learning assumption: 
\begin{assumption}[In distribution guarantee for IPO]\label{assump:learning_ipo}
We assume that the returned policy $\piipo$ satisfies that
\begin{align*}
  \piipo = \argmin_{\pi \in \Pi} \lipo(\pi),
\end{align*}
i.e., the returned policy $\piipo$ induces the smallest possible in-distribution error on its population loss.
\end{assumption}

With the setup, we can state and prove the formal version of the result:
\begin{proposition}[Formal version of of \pref{prop:ipo_partial}]
Denote $\piref$ as \emph{any} reference policy such that \pref{assump:global} breaks.
Let $\Pi_{\textsf{ipo}}$ be the set of IPO returned policies such that \pref{assump:learning_ipo} holds.
Then there exists policy $\pi \in \Pi_{\textsf{ipo}}$ such that $J(\pi) = -\infty$.
\end{proposition}

\begin{proof}
Without loss of generality, we consider a promptless setting, and assume that
the response space is $\Ycal = \{y_1,y_2,y_3\}$. Again without loss of generality,
we assume $\piref$ only covers $y_1$ and $y_2$, and thus \pref{assump:global} breaks.
Specifically, let $\piref(y_1) = \piref(y_2) = 1/2$. Then we have 
\begin{align*}
\piipo = \argmin_{\pi \in \Pi} \EE_{y^1,y^2 \sim \piref} \brk*{\prn*{\log \prn*{\frac{\pi(y^1)}{\pi(y^2)}} - I\prn*{y^1,y^2}/\beta}^2},
\end{align*}
which gives
\begin{align*}
\log \prn*{\frac{\piipo(y_1)}{\piipo(y_2)}} = p^\ast(y_1 \succ y_2)/\beta,
\end{align*}
and thus we have the relation that 
\begin{align*}
\piipo(y_1) = \piipo(y_2) \cdot \exp \prn*{p^\ast(y_1 \succ y_2)/\beta}.
\end{align*}
Let $\piipo(y_2) = \alpha \in (0,1]$, then for any $\alpha$
such that $\piipo(y_3) = 1 - (1+\exp \prn*{p^\ast(y_1 \succ y_2)/\beta}) \alpha > 0$,
we will have that $\KL\prn*{\piipo || \piref}$ is unbounded, and thus we complete the proof. 
\end{proof}

\subsection{DPO Has Vacuous Forward KL}\label{sec:dpo_forward}
In this section, we show that in the worst case, the forward KL of DPO is
vacuously large. 
We first see how we can relate the forward KL divergence of the DPO policy with
the reward learning guarantee. Consider any DPO policy $\pidpo$ and its corresponding reward model 
$\rdpo$. By construction of the DPO algorithm,
we have, for any $x,y$ pair that is covered in the dataset,
\neurips{$\pidpo(y \mid x) = \frac{\piref(y \mid x) \exp(\rdpo(x,y)/\beta)}{Z(x)},$}
\arxiv{
\begin{align*}
  \pidpo(y \mid x) = \frac{\piref(y \mid x) \exp(\rdpo(x,y)/\beta)}{Z(x)},
\end{align*}}
where $Z(x) = \sum_{y} \piref(y \mid x) \exp( \rdpo (x,y)/\beta)$.
Then the forward 
KL divergence is
\neurips{
{\small
\begin{align*}
\EE_{x,y \sim \rho\circ\piref} \brk*{\log \prn*{\frac{\piref(y \mid x)}{\pidpo(y \mid x)}}} = \EE_{x,y \sim \rho\circ\piref } \brk*{\log \prn*{ \frac{Z(x)}{\exp(\rdpo(x,y)/\beta)}}} =  \EE_{x,y \sim \rho\circ\piref} \brk*{ - \frac{\rdpo(x,y)}{\beta} + \log(Z(x))}.
\end{align*}
}
}
\arxiv{
\begin{align*}
  \KL(\piref||\pidpo) = \EE_{x \sim \rho} \brk*{\KL(\piref(\cdot \mid x) || \pidpo(\cdot \mid x))} &= \EE_{x \sim \rho} \EE_{y \sim \piref(\cdot \mid x)} \brk*{\log \prn*{\frac{\piref(y \mid x)}{\pidpo(y \mid x)}}} \\
  &= \EE_{x \sim \rho}  \EE_{y \sim \piref(\cdot \mid x)} \brk*{\log \prn*{ \frac{Z(x)}{\exp(\rdpo(x,y)/\beta)}}} \\
  &= \EE_{x \sim \rho} \EE_{y \sim \piref(\cdot \mid x)} \brk*{ - \frac{\rdpo(x,y)}{\beta} + \log(Z(x))}.
\end{align*}
}

Although the first term can be easily related to the reward learning guarantee, the second term
($\EE_{x \sim \rho} \brk*{\log(Z(x))}$) can unfortunately be vacuous without further assumptions.
We formalize in the following result: 
\begin{proposition}\label{prop:unbounded_forward_kl}
  There exist $\pidpo$ such that \pref{assump:reward_learning} holds, but 
  $\KL(\piref || \pidpo)$ is arbitrarily large.
\end{proposition}

\begin{proof}
First without loss of generality let us consider that $r^\ast > 0$.
Now suppose there exists $\tilde y$ such that $\piref(\tilde y \mid x) = \frac{1}{n^4}$ for all $x$, where $n$ will be determined soon. 
Now suppose that for all $x$,
$\rdpo(x,\tilde y) - r^\ast(x,\tilde y) = n$ and
$\rdpo(x, y) = r^\ast(x, y)$ for all $y \neq \tilde y$.
Now we can check that 
\begin{align*}
    \EE_{x \sim \rho} \EE_{y \sim \piref(\cdot \mid x)} \brk*{
    \prn*{\rdpo(x,y) - r^\ast(x,y) }^2} = \frac{1}{n^2},
\end{align*}
which is diminishing if we take $n$ to be big enough. We can also check that 
\begin{align*}
    \EE_{x \sim \rho} \EE_{y \sim \piref(\cdot \mid x)} \brk*{ - \frac{\rdpo(x,y)}{\beta}} \geq -\frac{1}{n^3 \beta} - \frac{R}{n^4 \beta}
\end{align*}
and thus the first term will have little impact on the final bound. 
However, the second term can be lower bounded as follows:
\begin{align*}
  \log \prn*{\sum_{y} \piref(y \mid x) \exp\prn*{\widehat  r(x,y)/\beta}} &= \log 
  \prn*{ \sum_{y} \piref(y \mid x) \exp\prn*{\frac{r^\ast(x,y) + \widehat  r(x,y) - r^\ast(x,y) }{\beta}}} \\
  &\geq \log 
  \prn*{ \sum_{y} \piref(y \mid x) \exp\prn*{\frac{\widehat  r(x,y) - r^\ast(x,y) }{\beta}}} \\
  &= \log 
  \prn*{  \piref(\tilde y \mid x) \exp\prn*{\frac{\widehat  r(x,\tilde y) - r^\ast(x,\tilde y) }{\beta}}} \\
  &= \frac{n}{\beta} - 4\log(n).
\end{align*}
Putting everything together, we have 
\begin{align*}
    \KL(\piref || \pidpo) \geq \frac{n}{\beta} - 4\log(n) - \frac{1}{n^3 \beta} - \frac{R}{n^4 \beta}
\end{align*}
and since we can take $n$ arbitrarily big we complete the proof.
\end{proof}
\section{Missing Proofs}\label{app:missing_proofs}

\subsection{Proof of \pref{prop:dpo_partial}}

\begin{proposition}[Restatement of \pref{prop:dpo_partial}]
Denote $\piref$ as \emph{any} reference policy such that \pref{assump:global} breaks. 
Let $\Pi_{\textsf{dpo}}$ be the set of DPO returned policies such that
\pref{assump:reward_learning} holds. Then there exists policy $\pi \in \Pi_{\textsf{dpo}}$
such that $J(\pi) = -\infty$.
\end{proposition}

\begin{proof}
 Again as in the proof sketch, without loss of generality, we consider a promptless setting, and assume that
  the response space is $\Ycal = \{y_1,y_2,y_3\}$. Again without loss of generality,
  we assume $\piref$ only covers $y_1$ and $y_2$, and thus \pref{assump:global} breaks. 
Now consider the optimal policy 

$$\pi^\ast(y) = \frac{\piref(y \mid x) \exp(r^\ast(y)/\beta)}{Z^\ast(t)}, 
\forall y \in \Ycal,$$
where $Z^\ast = \sum_{y \in \Ycal} \piref(y \mid x) \exp(r^\ast(y)/\beta)$, 
note that by construction $\pi^\ast(y_3) = 0$.
  
  Then consider the following policy $\pi$ such that
  \begin{align*}
    \beta \log \prn*{\frac{\pi(y_1)}{\piref(y_1)\cdot Z^\ast}} = r^\ast(y_1) - \sqrt{\edpo}, \mathand \beta \log\prn*{ \frac{\pi(y_2)}{\piref(y_2)\cdot Z^\ast}} = r^\ast(y_2) - \sqrt{\edpo},
  \end{align*}
  Then we have 
  \begin{align*}
 \EE_{y \sim  \piref} \brk*{\prn*{\beta \log \prn*{\frac{\pidpo(y)}{\piref(y \mid x) \cdot Z^\ast}}-r^\ast(x,y)}^2} = \edpo,
\end{align*}
  thus $\pi$ satisfies \pref{assump:reward_learning}.
  Rearranging we can see that $\pi(y_1) < \pi^\ast(y_1)$ and 
  $\pi(y_2) < \pi^\ast(y_2)$.
  Now since $\pi^\ast = 0$, we have 
  \begin{align*}
      \pi^\ast(y_1) + \pi^\ast(y_2) = 1,
  \end{align*}
  and combine we get $\pi(y_3) > 0$, which implies $\KL\prn*{\pi || \piref}$ is unbounded, since $\piref(y_3) = 0$.
\end{proof}
\subsection{Proof of \pref{thm:rlhf_partial}}
In this section we prove \pref{thm:rlhf_partial}:
\begin{theorem}[Restatement of \pref{thm:rlhf_partial}]
  Suppose that \pref{assump:reward_bounded} holds. Then for any reference policy $\piref$ such that 
  \pref{assump:klball} holds with $\ekl = \frac{2R'}{\beta}$, for any RLHF policy $\pirlhf$
  with $\widehat r$ such that (c.r. \pref{assump:reward_learning}),
  \begin{align*}
  \EE_{x,y \sim \rho \circ \piref} \brk*{ \prn*{ r^\ast(x,y) - \widehat r (x,y)}^2 }\leq \ereward,
  \end{align*} 
  or more generally, the event in \pref{lem:reward_relative} holds for $\widehat r$,
  we have 
  \begin{align*}
    J(\pi^\ast) - J(\pirlhf) \leq O(C_{\ekl} \sqrt{\ereward}).
  \end{align*}
\end{theorem}

To prove this we first prove the following lemma so we can 
leverage \pref{assump:klball}:

\begin{lemma}[Restatement of \pref{lem:rlhf_reverse_kl}]
    Suppose that \pref{assump:reward_bounded} holds. Then for any RLHF policy $\pirlhf$, we have that
    \begin{align*}
      \KL(\pirlhf || \piref) := 
      \EE_{x \sim \rho} \EE_{y \sim \pirlhf(\cdot \mid x)} \brk*{\log \prn*{\frac{\pirlhf(y \mid x)}{\piref(y \mid x)}}}
      \leq \frac{2R'}{\beta}.
    \end{align*}
\end{lemma}
\begin{proof}
since we have that $\pirlhf(y \mid x) = \frac{\piref(y \mid x) \exp(\hat r(x,y)/\beta)}{Z(x)}$
\emph{for all} $x \in \supp(\rho), y \in \Ycal$, we have
\begin{align*}
    \KL(\pirlhf || \piref) = \EE_{x \sim \rho} \EE_{y \sim \pirlhf(\cdot \mid x)} \brk*{\log \prn*{\frac{\exp\prn*{\hat r(x,y)}}{\beta Z(x)}}} = \EE_{x \sim \rho} \EE_{y \sim \pirlhf(\cdot \mid x)} \brk*{\frac{\hat r(x,y)}{\beta} - \log(Z(x))}.
\end{align*}
Plugging in the definition of $Z(x)$ we get 
\begin{align*}
    \log(Z(x)) = \log \prn*{\EE_{y \sim \piref(\cdot \mid x)} \brk*{\exp\prn*{\frac{\hat r(x,y)}{\beta}}}} \geq \EE_{y \sim \piref(\cdot \mid x)} \brk*{\frac{\hat r(x,y)}{\beta}}
\end{align*}
due to Jensen's inequality. Thus we have
\begin{align*}
    \KL(\pirlhf || \piref) \leq \EE_{x \sim \rho} \EE_{y \sim \pirlhf(\cdot \mid x)} \brk*{\frac{\hat r(x,y)}{\beta}} - \EE_{x \sim \rho} \EE_{y \sim \pirlhf(\cdot \mid x)} \brk*{\frac{\hat r(x,y)}{\beta}}  \leq  \frac{2R'}{\beta}.
\end{align*}
\end{proof}

Now with \pref{lem:rlhf_reverse_kl}, we can prove \pref{thm:rlhf_partial}:

\begin{proof}
    By \pref{lem:objective_decomposition}, we have 
\begin{align*}
    &~~~~J(\pi^\ast) - J(\pirlhf) \\&\leq 
    \EE_{x \sim \rho} \EE_{y^1 \sim \pi^\ast(\cdot \mid x), y^2 \sim \pirlhf(\cdot \mid x)} 
        \brk*{r^\ast(x,y^1) - \widehat r (x,y^1) - r^\ast(x,y^2) + \widehat r (x,y^2)} \\
    &\leq \sqrt{\EE_{x \sim \rho} \EE_{y^1 \sim \pi^\ast(\cdot \mid x), y^2 \sim \pirlhf(\cdot \mid x)} 
    \brk*{ \prn*{r^\ast(x,y^1) - \widehat r(x,y^1) - r^\ast(x,y^2) + \widehat r(x,y^2)}^2}} \\
    &\leq \sqrt{\cglobal^2 \EE_{x \sim \rho} \EE_{y^1, y^2 \sim \piref(\cdot \mid x)} 
    \brk*{ \prn*{r^\ast(x,y^1) - \widehat r(x,y^1) - r^\ast(x,y^2) + \widehat r(x,y^2)}^2}} \tag{\pref{lem:rlhf_reverse_kl} and \pref{assump:klball}}
    \\
    &\leq C\sqrt{\ereward}. \tag{\pref{lem:reward_relative}}
\end{align*}
\end{proof}
\neurips{
\section{Synthetic experiment for extrapolation} \label{sec:extrapolation_exp}
\subsection{Extrapolation with function approximation}
We first describe our experiment setup. We consider linear function 
approximation setting where we have 100 responses ($|\Ycal| = 100$).
We consider a 16-dimensional feature vector $\phi: \Ycal \to \RR^{16}$,
and we generate $\phi(y)$ by simply sampling 99 random 16-dimensional 
vectors where the $\ell_1$ norm of each vector is 1. We add one final 
$\phi(y) = [1,0,0,\dots]$.

We construct the implicit human reward $r^\ast (y) = {w^\ast}^{\top} \phi(y)$, where $w^\ast = [5,...]$, and the rest of the entries are 
sampled from Unif(-2,2). 

We parametrize the policies as softmax linear policies, i.e., 
we parametrize each policy $\pi$ with $w^\pi \in \RR^{16}$
such that $\pi(y) = \frac{{w^\pi}^\top \phi(y)}{\sum_{y \in \Ycal}
{w^\pi}^\top \phi(y)}$. One can check in this formulation the 
implicit reward in DPO ($\rdpo$) is linear in $\phi$.

We generate 10000 preference pairs, according to the BT model 
under $r^\ast$, for the first 50 responses. We checked that the 
first responses indeed span $\RR^{16}$. Thus the offline data has 
global coverage in linear function approximation setting. 

For on-policy RL methods, we first train a reward model. Then we 
simply perform gradient descent on the KL-regularized bandit loss 
(we assume $\piref$ is uniform). For DPO, we simply perform SGD 
on the offline preference dataset. We track two qualities over 
the training: the mean log probability of a random subset of preferred 
responses, and the log probability of best response $\phi(y) = [1,0,0,\dots]$. We plot the results in \pref{fig:extrapolation}.
We observe that both methods have the extrapolation behavior --
the probability of preferred responses decays but the probability of 
the optimal response goes up.

\begin{figure}
    \centering
    \includegraphics[width=0.4\linewidth]{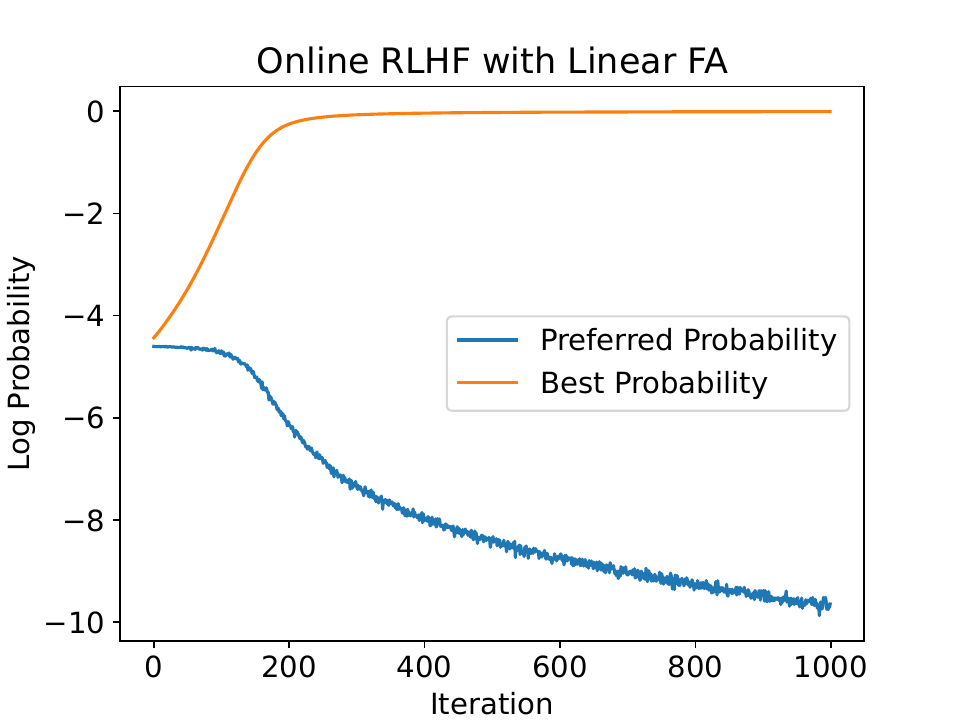}
    \includegraphics[width=0.4\linewidth]{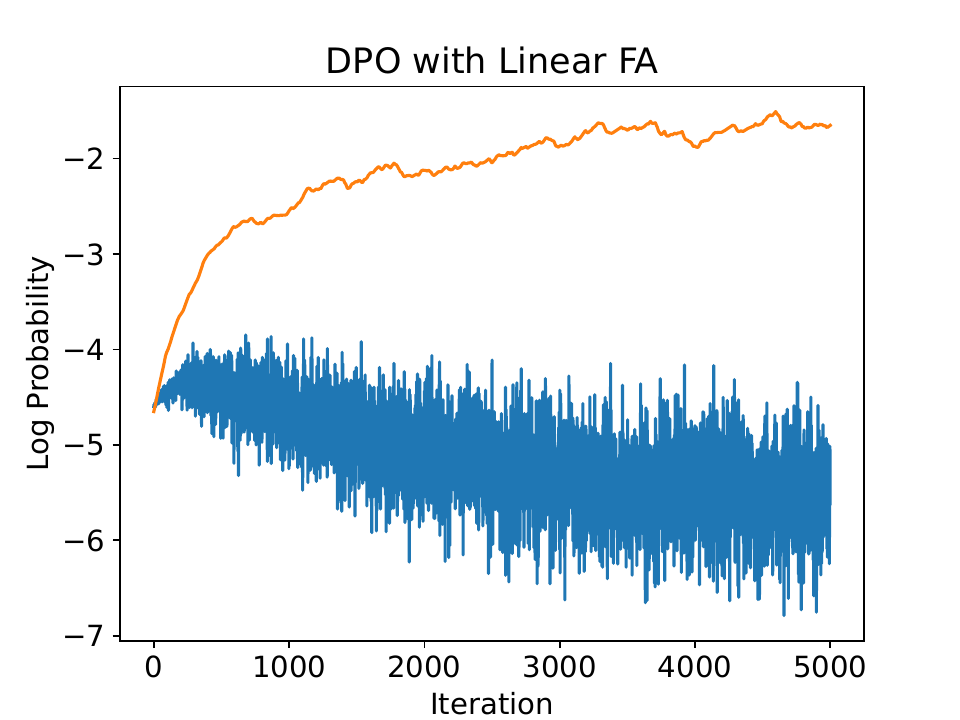}
    \caption{Extrapolation behavior of Online RL method and DPO under linear  function 
    approximation. We plot the mean log probability of the preferred
    responses and the log probability of the best response, which is 
    unseen in the training data. We see that both algorithms correctly 
    assigns increasing probability to the best response.}
    \label{fig:extrapolation}
\end{figure}

\subsection{Extrapolation without function approximation}
Now we describe the setting where function approximation fails, and this reduces 
to a Multi-arm bandit setting. We set $|\Ycal| = 500$, and the offline data only 
covers the first half of the responses. The $r^\ast(y)$ is set by sampling 
from Unif(-10,10), and we generate 10000 offline samples by uniformly sample 
pairs of responses from the first half of the response space, and then label 
them with BT model under $r^\ast$. We train DPO with 5000 iterations, and plot 
the mean probability of the responses \emph{outside} of the data support in \pref{fig:extrapolation_bad}: we observe that the mean probability of the out-of-distribution responses are increasing, however, this could be an undesirable behavior because the reward of the out-of-distribution responses could be 
arbitrarily bad. 

\begin{figure}
    \centering
    \includegraphics[width=0.4\linewidth]{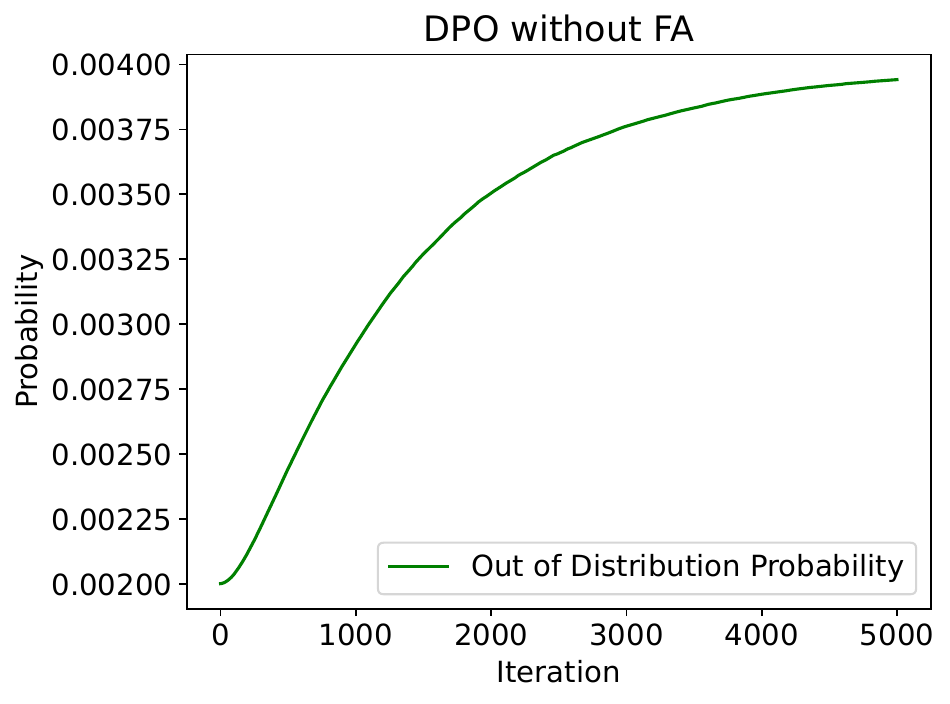}
    \caption{Extrapolation behavior of DPO without function approximation. We plot the average probability of out-of-distribution responses along the training and DPO assigns increasing probability to out-of-distribution responses.}
    \label{fig:extrapolation_bad}
\end{figure}
}
\section{Details of \pref{sec:hypo}} \label{sec:exp_details}

\subsection{Theoretical guarantee}\label{sec:hypo_theory}
In this section, we consider the constrained optimization version of HyPO 
(\pref{eq:constrained_dpo}). Note that the reward function class is identical 
to DPO, i.e.,
$\Rcal_{\mathsf{hypo}} = \crl*{\beta \log \prn*{\frac{\pi(y \mid x)}{\piref(y \mid x) Z(x)}} \mid \pi \in \Pi}$, where $Z(x)$ is the partition function. 
Then for each output policy $\pihypo$, we can denote its implicit reward function $\rhypo(x,y) := \beta\frac{\pihypo(y \mid x)}{\piref(y \mid x) \cdot Z(x)}$, and similarly to \pref{thm:rlhf_partial}, we can obtain the following 
guarantee in the partial coverage condition:

\begin{theorem}
  For any reference policy $\piref$ such that 
  \pref{assump:klball} holds with $\ekl = \frac{2R'}{\beta}$, for any HyPO policy $\pihypo$ such that the event in \pref{lem:reward_relative} holds,
  i.e., 
  \begin{align*}
      \EE_{x,y^1,y^2 \sim \mu \circ \piref} \brk*{ \prn*{ r^\ast(x,y^1) - r^\ast(x,y^2) - \rhypo (x,y^1) + \rhypo(x,y^2)}^2 }\leq \ehypo,
    \end{align*} 
  we have
  \begin{align*}
    J(\pi^\ast) - J(\pihypo) \leq O(C_{\ekl} \sqrt{\ehypo}).
  \end{align*}
\end{theorem}
\begin{proof}
The proof mostly follows the proof of \pref{thm:rlhf_partial}.
It remains to show the following two properties:

1) Note that \pref{thm:rlhf_partial} requires  \pref{assump:reward_bounded}, which does not hold for 
$\rhypo$ (note that $\rhypo$ is only bounded under $\rho$, 
but not for all $x$), but we only use it to prove the sufficient 
condition in \pref{lem:rlhf_reverse_kl}, which is satisfied 
by the constraint of HyPO. 

2) We need to check that the premise of 
\pref{lem:objective_decomposition} holds, i.e., 
\begin{align*}
        \pihypo \in \argmax_{\pi} \EE_{x \sim \rho} \brk*{ \EE_{y \sim \pi(\cdot \mid x)} [\rhypo(x,y)] - \beta \KL\prn*{\pi(\cdot \mid x) || \piref(\cdot \mid x)}},
\end{align*}
note that with the reparametrization between $\pihypo$ and $\rhypo$,
$\pihypo$ is always among the minimizer of the \emph{unconstrained}
policy set, so we can still invoke \pref{lem:reward_relative}.
The rest of the proof now follows the proof of \pref{thm:rlhf_partial}
so we omit the details.\loose 
\end{proof}

Finally, we remark the connection to the negative result of DPO, i.e, 
\pref{prop:dpo_partial}: note that given $\KL(\pihypo || \piref) \leq \infty$,
we have that for all $x$ such that $\rho(x) > 0$, we have for all $y$,
$\beta \log \prn*{ \frac{\pihypo(y \mid x)}{\piref(y \mid x)}} < \infty$, (again with 
the convention that $\frac{0}{0} = 0$), which breaks the construction 
of \pref{prop:dpo_partial}. 

\subsection{Experiment details}

\subsubsection{Summarization}
In this section, we provide more details of our summarization experiment. We use the Pythia 1.4B and 2.8B model \citep{biderman2023pythia} with hugging face model cards: EleutherAI/pythia-1.4b-deduped and EleutherAI/pythia-2.8b-deduped. The TL;DR dataset is available at \url{ https://github.com/openai/summarize-from-feedback}. The human reference dataset contains 117k training, 6.45K validation and 6.55K testing data. The preference dataset contains 92.9K training and 83.8K validation data. The reward evaluation 
and KL computation is performed on the whole validation data of the reference dataset. The GPT winrate is computed on a subset of 600 samples from the validation data. The GPT API checkpoint we use is gpt-4-0613. We follow the standard prompt for the winrate evaluation (e.g., see Appendix D.3 of \citet{gao2024rebel}). Below we provide the hyperparameter for HyPO and DPO. Note that to optmize the online KL, we use Reinforce with Leave One Out (RLOO) \citep{kool2019buy} with two generations per prompt  ($k=2$) and optimize trajectory-level KL. 

For our experiment, we run on a cluster of mixture of Nvidia A6000
and L40 GPUs with 48 GB VRAM. We use 4 GPUs in parallel for training,
and for DPO the experiment time varies from 1 hour to 2 hours to finish, and for HyPO the time varies between 4 hours to 5 hours. 

\begin{table}[h]
\centering
\begin{minipage}{0.45\textwidth}
\centering
    \caption{RM/SFT hyperparameters.}
  \begin{tabular}{cc}
    \toprule
    Learning rate & 3e-6 \\
    Batch size & 64  \\
    Learning rate scheduler & cosine  \\
    Optimizer & Adamw \\
    LoRA & False \\
    \bottomrule
  \end{tabular}\label{table:rm}
\end{minipage}\hfill
\begin{minipage}{0.45\textwidth}
    \centering
    \caption{DPO hyperparameters.}
  \begin{tabular}{cc}
    \toprule
    Learning rate & 3e-6 \\
    Batch size & 64  \\
    Learning rate scheduler & cosine  \\
    Optimizer & Adamw \\
    $\beta$ & 0.05 \\
    \bottomrule
  \end{tabular}\label{table:dpo}
\end{minipage}
\end{table}

\begin{table}[h]
\centering
\begin{minipage}{0.45\textwidth}
\centering
    \caption{HyPO hyperparameters.}
  \begin{tabular}{cc}
    \toprule
    Learning rate & 3e-6 \\
    Batch size & 64  \\
    Learning rate scheduler & cosine  \\
    Optimizer & Adamw \\
    $\beta$ & 0.05 \\
    $\lambda$ & 0.0001 \\
    RLOO $k$ & 2 \\
    \bottomrule
  \end{tabular}\label{table:hypo}
\end{minipage}\hfill
\begin{minipage}{0.45\textwidth}
    \centering
    \caption{Lora configurations.}
  \begin{tabular}{cc}
    \toprule
    $r$ & 1024 \\
    $\alpha$ & 2048  \\
    Dropout & 0  \\
    \bottomrule
  \end{tabular}\label{table:lora}
\end{minipage}
\end{table}

\subsubsection{General Chat}
For the base model of general chat experiments, we use Llama3-8B-Instruct \citep{meta2024introducing} with hugging face model card: meta-llama/Meta-Llama-3-8B-Instruct. The dataset card of the Ultrafeedback dataset \citep{cui2023ultrafeedback}
is HuggingFaceH4/ultrafeedback\_binarized. In addition to the KL penalty, in the general chat task we add an additional length penalty, and the online penalty of a generation $y$ with context $x$ becomes $\log\prn*{\frac{\pi(y \mid x)}{\piref(y \mid x)}} + \alpha |y|$. We summarize the hyperparameter of each baseline below. 

We run the general chat experiment on a node of 8 Nvidia A100 80GB GPUs. DPO takes 3 hours to train one epoch while HyPO takes 18 hours to train one epoch. 
\begin{table}[h]
\centering
\begin{minipage}{0.45\textwidth}
\centering
    \caption{HyPO hyperparameters.}
  \begin{tabular}{cc}
    \toprule
    Learning rate & 3e-6 \\
    Batch size & 8  \\
    Learning rate scheduler & linear  \\
    Optimizer & Adamw \\
    $\beta$ & 0.05 \\
    $\lambda$ & 0.0002 \\
    RLOO $k$ & 2 \\
    $\alpha$ & 0.02 \\
    \bottomrule
  \end{tabular}\label{table:hypo_uf}
\end{minipage}\hfill
\begin{minipage}{0.45\textwidth}
    \centering
    \caption{DPO hyperparameters.}
  \begin{tabular}{cc}
    \toprule
    Learning rate & 3e-6 \\
    Batch size & 8  \\
    Learning rate scheduler & linear  \\
    Optimizer & Adamw \\
    $\beta$ & 0.05 \\
    \bottomrule
  \end{tabular}\label{table:dpo_uf}
\end{minipage}
\end{table}

\neurips{
\clearpage
\newpage

\section*{NeurIPS Paper Checklist}

\begin{enumerate}

\item {\bf Claims}
    \item[] Question: Do the main claims made in the abstract and introduction accurately reflect the paper's contributions and scope?
    \item[] Answer: \answerYes{} 
    \item[] Justification: We have theoretical and empirical results supporting the claims made in the abstract and introduction.

\item {\bf Limitations}
    \item[] Question: Does the paper discuss the limitations of the work performed by the authors?
    \item[] Answer: \answerYes{} 
    \item[] Justification: We provide a discussion about
    the limitations of our work in the discussion section.

\item {\bf Theory Assumptions and Proofs}
    \item[] Question: For each theoretical result, does the paper provide the full set of assumptions and a complete (and correct) proof?
    \item[] Answer: \answerYes{} 
    \item[] Justification: We clearly state all our assumptions 
    in all theorem statements, and the proofs can be found in the
    appendix. 

    \item {\bf Experimental Result Reproducibility}
    \item[] Question: Does the paper fully disclose all the information needed to reproduce the main experimental results of the paper to the extent that it affects the main claims and/or conclusions of the paper (regardless of whether the code and data are provided or not)?
    \item[] Answer: \answerYes{} 
    \item[] Justification: We provide pseudocode, hyparameter table
    and code in this submission.

\item {\bf Open access to data and code}
    \item[] Question: Does the paper provide open access to the data and code, with sufficient instructions to faithfully reproduce the main experimental results, as described in supplemental material?
    \item[] Answer: \answerYes{} 
    \item[] Justification: We use public data on hugging face, 
    and we submit the code for this submission.

\item {\bf Experimental Setting/Details}
    \item[] Question: Does the paper specify all the training and test details (e.g., data splits, hyperparameters, how they were chosen, type of optimizer, etc.) necessary to understand the results?
    \item[] Answer: \answerYes{} 
    \item[] Justification: We provide all the training details 
    and hyperparameters. 

\item {\bf Experiment Statistical Significance}
    \item[] Question: Does the paper report error bars suitably and correctly defined or other appropriate information about the statistical significance of the experiments?
    \item[] Answer: \answerYes{} 
    \item[] Justification: We repeat all our experiments over 3 seeds.

\item {\bf Experiments Compute Resources}
    \item[] Question: For each experiment, does the paper provide sufficient information on the computer resources (type of compute workers, memory, time of execution) needed to reproduce the experiments?
    \item[] Answer: \answerYes{} 
    \item[] Justification: We provided run time and type of GPUs
    we used in our experiments. 
    
\item {\bf Code Of Ethics}
    \item[] Question: Does the research conducted in the paper conform, in every respect, with the NeurIPS Code of Ethics \url{https://neurips.cc/public/EthicsGuidelines}?
    \item[] Answer: \answerYes{} 
    \item[] Justification: We respect the code of conduct. 

\item {\bf Broader Impacts}
    \item[] Question: Does the paper discuss both potential positive societal impacts and negative societal impacts of the work performed?
    \item[] Answer: \answerYes{} 
    \item[] Justification: We discuss the broader impact in the discussion section. 
    
\item {\bf Safeguards}
    \item[] Question: Does the paper describe safeguards that have been put in place for responsible release of data or models that have a high risk for misuse (e.g., pretrained language models, image generators, or scraped datasets)?
    \item[] Answer: \answerNA{} 
    \item[] Justification: We do not release data or models in this work

\item {\bf Licenses for existing assets}
    \item[] Question: Are the creators or original owners of assets (e.g., code, data, models), used in the paper, properly credited and are the license and terms of use explicitly mentioned and properly respected?
    \item[] Answer: \answerYes{} 
    \item[] Justification: We cited all the dataset and models used in this work. 
    \begin{itemize}
        \item The answer NA means that the paper does not use existing assets.
        \item The authors should cite the original paper that produced the code package or dataset.
        \item The authors should state which version of the asset is used and, if possible, include a URL.
        \item The name of the license (e.g., CC-BY 4.0) should be included for each asset.
        \item For scraped data from a particular source (e.g., website), the copyright and terms of service of that source should be provided.
        \item If assets are released, the license, copyright information, and terms of use in the package should be provided. For popular datasets, \url{paperswithcode.com/datasets} has curated licenses for some datasets. Their licensing guide can help determine the license of a dataset.
        \item For existing datasets that are re-packaged, both the original license and the license of the derived asset (if it has changed) should be provided.
        \item If this information is not available online, the authors are encouraged to reach out to the asset's creators.
    \end{itemize}

\item {\bf New Assets}
    \item[] Question: Are new assets introduced in the paper well documented and is the documentation provided alongside the assets?
    \item[] Answer: \answerNA{} 
    \item[] Justification: We do not release any new asset. 

\item {\bf Crowdsourcing and Research with Human Subjects}
    \item[] Question: For crowdsourcing experiments and research with human subjects, does the paper include the full text of instructions given to participants and screenshots, if applicable, as well as details about compensation (if any)? 
    \item[] Answer: \answerNA{} 
    \item[] Justification: We do not have human subject. 

\item {\bf Institutional Review Board (IRB) Approvals or Equivalent for Research with Human Subjects}
    \item[] Question: Does the paper describe potential risks incurred by study participants, whether such risks were disclosed to the subjects, and whether Institutional Review Board (IRB) approvals (or an equivalent approval/review based on the requirements of your country or institution) were obtained?
    \item[] Answer: \answerNA{} 
    \item[] Justification: No human subject. 

\end{enumerate}

}
\end{document}